\documentclass[twoside]{article}

\renewcommand{\circ}{\odot}
\renewcommand{\vec}{\bar}
\usepackage[accepted]{aistats2026}
\usepackage{times}  % DO NOT CHANGE THIS
\usepackage{helvet}  % DO NOT CHANGE THIS
\usepackage{courier}  % DO NOT CHANGE THIS
\usepackage[hyphens]{url}  % DO NOT CHANGE THIS
\usepackage{graphicx} % DO NOT CHANGE THIS
\usepackage{natbib}  % DO NOT CHANGE THIS AND DO NOT ADD ANY OPTIONS TO IT
\usepackage{caption} % DO NOT CHANGE THIS AND DO NOT ADD ANY OPTIONS TO IT
\usepackage{algorithm}
\usepackage{algorithmic}

\usepackage{newfloat}
\usepackage{listings}
\usepackage[utf8]{inputenc} % allow utf-8 input
\usepackage{booktabs}       % professional-quality tables
\usepackage{amsfonts}       % blackboard math symbols
\usepackage{nicefrac}       % compact symbols for 1/2, etc.
\usepackage{microtype}      % microtypography
\usepackage[dvipsnames]{xcolor}         % colors
\usepackage{amsmath,amsthm,amsfonts,amssymb}

\usepackage{glossaries}
\usepackage{siunitx}
\usepackage{graphicx}
\usepackage{tcolorbox}

\usepackage{multicol}
\usepackage{multirow}
\usepackage{pifont}

\usepackage{caption}
\usepackage{subcaption}

\usepackage{enumitem}

\usepackage{adjustbox}

\usepackage[normalem]{ulem}

\newcommand{\expect}{\mathbb{E}}

\newcommand{\py}{\mathbf{y}}
\newcommand{\pQ}{\mathbf{Q}}
\newcommand{\pL}{\mathbf{L}}
\newcommand{\pR}{\mathbf{R}}
\newcommand{\pA}{\mathbf{A}}

\newcommand{\pE}{\mathbf{E}}
\newcommand{\pM}{\mathbf{M}}

\newcommand{\pphi}{\boldsymbol{\phi}}
\newcommand{\ppsi}{\boldsymbol{\psi}}

\newcommand{\pxi}{\boldsymbol{\xi}}

\newcommand{\pv}{\mathbf{v}}

\newcommand{\px}{\mathbf{x}}

\newcommand{\pz}{\mathbf{z}}
\newcommand{\pb}{\mathbf{b}}

\newcommand{\tf}{\tilde{f}}

\newcommand{\erf }{\text{erf}}

\newcommand{\reals}{\mathbb{R}}
\newcommand{\uspace}{\mathbb{U}}
\DeclareFontFamily{U}{mathx}{\hyphenchar\font45}
\DeclareFontShape{U}{mathx}{m}{n}{
      <5> <6> <7> <8> <9> <10>
      <10.95> <12> <14.4> <17.28> <20.74> <24.88>
      mathx10
      }{}
\DeclareSymbolFont{mathx}{U}{mathx}{m}{n}
\DeclareMathSymbol{\bigtimes}{1}{mathx}{"91}
\newcommand{\sos}{SoS}
\newcommand{\fspace}{\mathcal{C}}
\newcommand{\sosfuncs}{\Sigma[\fspace]}
\newcommand{\psd}{\succeq 0}
\newcommand{\posdef}{\succ 0}
\newcommand{\finnerprod}[2]{\big\langle #1, #2 \big\rangle}

\newcommand{\diag}{\text{diag}}

\newcommand{\edt}[1]{\textcolor{black}{#1}}
\newcommand{\rmv}[1]{}
\theoremstyle{definition}
\newtheorem{definition}{Definition}
\newtheorem{problem}{Problem}
\newtheorem{assumption}{Assumption}

\theoremstyle{plain}
\newtheorem{theorem}{Theorem}

\newtheorem{lemma}{Lemma}

\theoremstyle{plain}
\newtheorem{remark}{Remark}

\begin{document}



\twocolumn[

\aistatstitle{Learning Markov Processes as Sum-of-Square Forms for Analytical Belief Propagation}

\aistatsauthor{ Peter Amorese \And Morteza Lahijanian}

\aistatsaddress{ University of Colorado Boulder \And  University of Colorado Boulder} ]

\begin{abstract}
Harnessing the predictive capability of Markov process models requires propagating probability density functions (beliefs) through the model. For many existing models however, belief propagation is analytically infeasible, requiring approximation or sampling to generate predictions. This paper proposes a functional modeling framework leveraging sparse Sum-of-Squares (SoS) forms for valid (conditional) density estimation. We study the theoretical restrictions of modeling conditional densities using the SoS form, and propose a novel functional form for addressing such limitations.
The proposed architecture enables generalized simultaneous learning of basis functions and coefficients, while preserving analytical belief propagation.
In addition, we propose a training method that allows for exact adherence to the normalization and non-negativity constraints. Our results show that the proposed method achieves accuracy comparable to state-of-the-art approaches while requiring significantly less memory in low-dimensional spaces, and it further scales to 12D systems when existing methods fail beyond 2D.
\end{abstract}

\section{INTRODUCTION}
Machine learning models offer a flexible and useful means of predicting real world processes. Due to either uncertainty, or stochastic environmental disturbances, such processes are often stochastic in nature. Models that account for this stochasticity, can significantly improve the quality and robustness of predictions, making them more reliable. However, propagating uncertainty in the state of the process through the model, i.e., \textit{belief propagation}, is often analytically infeasible for continuous state-spaces, thus requiring sampling or approximations. Even with the most accurate models, the need for approximation can cause prediction quality and interpretability to suffer. This work studies a novel functional form for learning general stochastic processes, while permitting analytical belief propagation. 
% The approach leverages sparsity in functional optimization to scale to higher dimensions and maintain representational capacity. 

Belief propagation is the process of computing the predicted marginal distribution (belief) of the state of the system at a future time. For continuous state systems, a belief belongs to a \textit{function-space} which is generally infinite-dimensional. Certain simple processes, in particular, linear systems with additive Gaussian noise, permit analytical belief propagation. However, such systems are often not adequate for describing many real world processes. Deviation from either the linear assumption or the additive Gaussian noise assumption typically makes belief propagation subject to intractable integrals \citep{jasour2021moment}. 

% \qh{this paragraph needs to not just list the methods, needs overall limitation}
The standard means of mitigating this issue is through approximation. A prominent example is non-linear filtering, in which linearization-based methods \citep{schei1997finite} (such as the Extended Kalman Filter) and second-order approximations have been proposed \citep{julier2004unscented}. Building upon such foundational approaches, Gaussian Mixture Model- (GMM) based approaches account for multi-modal approximation of the belief \citep{alspach2003nonlinear, figueiredo2024uncertainty,kulik2024nonlinearity}. 
% \qh{remove this following paragraph, add kulik2024nonlinearity to the previous citation?} Here, component-splitting techniques can improve the fidelity of GMM predictions, at the cost of adding (exponentially) more mixture components each timestep \citep{kulik2024nonlinearity}. 
Not only are these methods approximate, they make restrictive assumptions about the underlying Markov process, such as additive Gaussian noise.

In contrast, sampling-based methods, such as particle filters \citep{arulampalam2002tutorial}, can leverage generative models \citep{jonschkowski2018differentiable} and propagate particles through the model to obtain a particle approximation of the belief. Nonetheless, these methods often require a large number of samples for high-dimensional systems. Furthermore, reconstructing a probability density from a particle set involves post-processing through density estimation techniques, e.g., Kernel Density Estimation \citep{parzen1962estimation}. Other approaches, such as moment-based methods \citep{jasour2021moment} can exactly propagate the moments of a belief for a class of dynamical systems; however, moments are generally insufficient for reconstructing the full belief \citep{akhiezer2020classical}. 
% \qh{Something like: These approximation methods either cannot propagate beliefs tractably or analytically}

\citet{amorese2026universal} recently proposed a method of solving the modeling problem for analytical belief propagation. The Markov process is modeled using Bernstein-polynomial Normalizing Flows (BNF), leveraging the flexible analytical characteristics of polynomials to perform belief propagation. The Bernstein polynomial basis is a \textit{partition of unity} (i.e., a set of non-negative functions that sum to one everywhere), a unique and rare characteristic that allows Bernstein Normalizing Flow models to model valid Markov processes. Unfortunately, by definition, the Bernstein basis is \textit{dense}, i.e., the number of required non-zero parameters is exponential in the dimension of the system. Consequently, the time cost of computing future beliefs is also exponential, preventing BNF from scaling beyond two-dimensional systems.

Sum-of-squares (\sos{}) forms are a powerful functional for modeling non-negativity, and have been used for tractable (conditional) density modeling \citep{loconte2025sum, rudi2021psd}, and have been shown to be more expressive than mixture models \citep{marteau2020non, loconte2023subtractive}. Modeling Markov processes using the proposed conditional density model (e.g. in \cite{rudi2021psd}) yields a rational function, which does not permit analytical belief propagation.

This work studies a novel functional form for modeling general Markov processes using \sos{} forms that, by construction, permits analytical belief propagation. 
We leverage Sum-of-Squares theory to form a valid \textit{conditional} density estimator for which one can train not only the coefficients, but also the \textit{basis functions} themselves.
%allow for optimization of the basis functions themselves. 
Similar to mixture models, this freedom allows the optimization to intelligently tune and allocate basis functions to achieve a sparse and accurate representation of the underlying process. 
We study the theoretical implications, limitations, and advantages of the \sos{} functional form, substantiated through experimental results.

The key contributions of this work are as follows:
\begin{itemize}
    \item a theoretical analysis of the consequences of using \sos{} for conditional density estimation, proving an important result about its restrictions under mild assumptions, 
    \item a novel functional form that alleviates the restrictions of \sos{}, without sacrificing any desireable theoretical attributes,
    \item a means of enforcing the valid-distribution constraints through direct parameterization, and
    \item experimental comparisons and demonstrations of the efficacy of the proposed approach for modeling high-dimensional systems.
\end{itemize}

\subsection*{Notation}
% The open unit box is denoted $\uspace^d = (0, 1)^d$.
% Vectors are denoted in bold, e.g., $\px \in \uspace^d$.  The $i$-th element of vector $\px$ is denoted by $x_i$.
% The set of all multivariate functions that are continuous and real-valued on $\uspace^d$ as $\fspace(\uspace^d) \subset \fspace(\reals^d)$.
% Probability density functions over random vectors $\px$ are denoted as $p(\px)$.
% A Positive Semi-Definite (or Positive Definite) matrix $\pA$ is denoted as $\pA \psd$ (resp. $\pA \posdef$). 
% The inner product of two functions $f, g$ is denoted 

We denote the open unit box by $\uspace^d = (0,1)^d$.
Vectors are written in bold, e.g., $\px \in \uspace^d$, with the $i$-th element of $\px$ denoted by $x_i$.
The set of all real-valued continuous multivariate functions on $\uspace^d$ is denoted by $\fspace(\uspace^d) \subset \fspace(\reals^d)$.
Probability density functions over random vectors $\px$ are written as $p(\px)$.
A positive semi-definite (resp. positive definite) matrix $\pA$ is denoted by $\pA \psd$ (resp. $\pA \posdef$).
The inner product of two functions $f,g$ is defined as
$\langle f, g \rangle = \int_{\uspace^d} f(\px) g(\px) d\px$.
The Hadamard (element-wise) product of two matrices (or vectors) $\pA$ and $\mathbf{B}$ is denoted by $\pA \circ \mathbf{B}$.
% The Hadamard (element-wise) product of two matrices (or vectors) $\pA$ and $\mathbf{B}$ is denoted $\pA \circ \mathbf{B}$.
\section{PROBLEM FORMULATION} \label{sec:pf}

In this work, we consider a general discrete-time stochastic process that evolves in $\reals^d$. Our goal is to study the propagation of its Probability Density Function (PDF), also referred to as the \emph{belief}, solely by using data. For simplicity of presentation and without loss of generality, we assume that the state space is transformed to the open unit box $\uspace^d$.\footnote{Any random variable in $\reals^d$ can be mapped to $\uspace^d$ via diffeomorphisms (such as $\erf(\cdot)$) without inhibiting integration capability. For details, see \cite{amorese2026universal}.}

Let $\px_k \in \uspace^d$ denote the (transformed) state of the stochastic process at time step $k \in \mathbb{N}^0$, with associated PDF (belief) $p(\px_k)$. Furthermore, let $\vec{\px} = \px_0, \px_1, \ldots, \px_K$ with $K \in \mathbb{N}^0$ represent a sequence of states generated by the process. We assume that the process is Markov, i.e., the state evolution is conditionally dependent only on the previous state, so that transitions follow the time-invariant distribution $p(\px_k \mid \px_{k-1})$ (also denoted $p(\px' \mid \px)$).
Consequently, the process is uniquely defined by the initial belief $p(\px_0)$ and the state-transition distribution $p(\px_k \mid \px_{k-1})$ for $k = 1,\ldots,K$.

In this work, we focus on the scenarios where the Markov process is \emph{unknown} and must be learned from data. To this end, we assume two sets of data are given: initial state data $\mathcal{D}_{\px_0} = \{(\hat{\px}_0)_i\}_{i=1}^{N_0}$ and state-transition data $\mathcal{D}_{\px'} = \{(\hat{\px}, \hat{\px}')_i\}_{i=1}^N$, where $\hat\px \in \uspace^d$ is a point and $\hat\px' \in \uspace^d$ is a realization of the one-step evolution of the process from $\hat\px$.

To learn the process, we consider a parametric density function $p_\theta(\px_0)$ and conditional density function $p_\theta(\px_k \mid \px_{k-1})$, where parameters $\theta$ must be learned from $\mathcal{D}_{\px_0}$ and $\mathcal{D}_{\px'}$.
Equipped with $p_\theta(\px_0)$ and $p_\theta(\px_k \mid \px_{k-1})$, inference of future beliefs can be performed via recursive \textit{belief propagation}:
\begin{align} \label{eq:marg}
    p_\theta(\px_k) = \int_{\px_{k-1}} p_\theta(\px_k \mid \px_{k-1}) p_\theta(\px_{k-1}) d\px_{k-1}
\end{align}
starting from $p_\theta(\px_0)$. However, for many classes of conditional density estimators, the integral in Equation~\eqref{eq:marg} is analytically intractable. 

This paper tackles this challenge by choosing a functional form for learning $p_\theta(\px_0)$ and $p_\theta(\px' \mid \px)$ subject to the following three criteria:
\begin{enumerate} [label=C\arabic*]
    \item\hspace{-2mm}. \textbf{Representational capacity}: with enough parameters, the distributions can approximate the true distribution $p(\px_0)$ and conditional distribution $p(\px' \mid \px)$ arbitrary well. \label{cond:universal}
    \item\hspace{-2mm}. \textbf{Analytical belief propagation}: Equation~\eqref{eq:marg} can be computed exactly. \label{cond:analytical}
    \item\hspace{-2mm}. \textbf{Sparse parameter representation}: the model can be stored with polynomial memory complexity with respect to the dimension $d$. \label{cond:sparse}
\end{enumerate}

This Markov Process learning problem can be formally stated as follows.
\begin{problem} \label{prob}
    Given datasets $\mathcal{D}_{\px_0}$ and $\mathcal{D}_{\px'}$ generated by a Markov process and a time horizon $K \in \mathbb{N}$, learn $p_\theta(\px_0)$ and $p_\theta( \px' \mid \px)$ that adheres to Conditions~\ref{cond:universal}-~\ref{cond:sparse} and compute $p_\theta(\px_K)$.
\end{problem}

Describing a class of models that adheres to all three conditions is a very challenging problem. For instance, Gaussian-mixture models with linear conditional dependence satisfy~\ref{cond:analytical} and~\ref{cond:sparse} at the cost of very limited representational capacity. Generative models satisfy~\ref{cond:universal} and~\ref{cond:sparse}, but require approximation or sampling to generate predictions. Lastly, dense Bernstein polynomial models satisfy~\ref{cond:universal} and~\ref{cond:analytical}, but struggle to scale due to exponential blow-up in parameters, violating \ref{cond:sparse}.

%Ensuring that analytical belief propagation is feasible requires special functional forms.
Problem \ref{prob} is posed as a constrained functional optimization problem. Our approach is to restrict the class of functions to  expressive Sum-of-Squares (\sos{}) forms that emit tractable normalization and non-negativity constraints as well as analytical belief propagation. In Section~\ref{sec:background}, we provide an overview of \sos{} functions and their properties and then detail our approach in Section~\ref{sec:approach}.

For the remainder of the paper, we drop the subscript $\theta$ and assume densities $p(\cdot)$ refer to the parameterized model.
%A detailed proof of Lemma \ref{lem:fixed_point} is provided in the Appendix.

%Specifically, a valid Conditional Density Estimator (CDE) $f(\px_{k-1}, \px_k) = p_\theta(\px_k | \px_{k-1})$ must adhere to the following two properties:
%\begin{align}
%    f(\px, \py) &\geq 0 \; &\forall \px, \py \in \uspace^d \label{eq:nonneg} \\
%    \int_{\uspace^d} f(\px, \py) d\py &= 1 \; &\forall \px \in \uspace^d. \label{eq:norm} 
%\end{align}

%\pa{Talk about challenges}
%i) analytical belief propagation and ii) sparse parameter representation.

\section{BACKGROUND: SUM-OF-SQUARES}
\label{sec:background}
To provide the necessary background for our methodology, we first review the key principles of Sum-of-Squares theory. Consider a vector of $n$ basis functions $\pb(\px) = [b_1(\px), ..., b_n(\px)]^T \in \fspace(\uspace^d)^n$. 

\begin{definition}[Sum-of-Squares]
    A function $f \in \fspace(\uspace^d)$ is \emph{Sum-of-Squares} iff $f$ can be written as $f(\px) = \pb^T(\px) \pQ \pb(\px)$ such that $\pQ \in \reals^{d\times d}$ is positive semi-definite (PSD), i.e., $\pQ \psd$. The set of all \sos{} functions is denoted $\sosfuncs \subset \fspace(\uspace^d)$.
\end{definition}

It is straightforward to show that if $f \in \sosfuncs$ then $f \geq 0$ for all $\px \in \uspace^d$.
Specifically, since $\pQ \psd$, it can be factored as $\pQ = \pL^T \pL$ where $\pL \in \reals^{d \times d}$. Then, $f(\px) = (\pL\pb(\px))^T (\pL\pb(\px))$, which is the inner product of a vector with itself.

Ensuring that a function is non-negative over a domain is a NP-hard problem in general \citep{parrilo2003semidefinite}, making functional optimization over non-negative functions intractable.
\sos{} forms offer a \textit{relaxation} of non-negative functions via a tractable PSD constraint. Additionally, \sos{} forms possess rich representational properties for some choices of basis functions \citep{putinar1993positive, nie2007complexity}. Given a fixed set of basis functions, optimization over non-negative functions is relaxed to optimization over PSD matrices, as is done in Semi-Definite Programming (SDP) \citep{vandenberghe1996semidefinite}. We employ techniques from \sos{} theory to optimize over density functions, where ensuring non-negativity is essential.

%\subsection{Conditional Density Estimation}
%This section reviews the challenges of designing a conditional density estimator, with a focus on partition of unity. Show that partition of unity is required. Can also discuss why Bernstein fits this mold. \pa{part of PF}

\section{SUM-OF-SQUARE FORM DENSITY ESTIMATION}
\label{sec:approach}
%Discuss the proposed \sos{} form for functional modeling - discuss non-negativity properties, analytical integrability properties, and PSD conditions.

To approach Problem~\ref{prob}, 
we propose models of the (conditional) density estimators as \sos{} forms. The benefits of this are three-fold. Firstly, \sos{} forms allow for rich non-negative functional representation (\ref{cond:universal}). Second, with basis functions belonging to a class of functions with attractive closed-form-integrable properties, belief propagation in Equation~\eqref{eq:marg} can be performed analytically (\ref{cond:analytical}). Lastly, general \sos{} forms do not put any restrictions (such as non-negativity) on the type or quantity of basis functions themselves, allowing for sparse representations (\ref{cond:sparse}). 

For the remainder of this section, we focus on conditional density estimators (CDE). Let $f : \uspace^d \times \uspace^d \rightarrow \reals_{\geq 0}$
denote a (multivariate) function representation of the conditional distribution $f(\px, \px') = p(\px' \mid \px)$. For $f$ to be a valid conditional distribution, two fundamental properties must hold:
\begin{align}
    f(\px, \px') &\geq 0 &&\forall \px, \px' \in \uspace^d, \label{eq:nonneg} \\
    \int_{\uspace^d} f(\px, \px') d\px' &= 1  &&\forall \px \in \uspace^d. \label{eq:norm} 
\end{align}
Let $f$ be a \sos{} form 
\begin{equation} \label{eq:fsos}
    f(\px, \px') = \pb^T(\px, \px') \; \pQ \; \pb(\px, \px').
\end{equation}
By construction, if $\pQ \psd$, then Equation~\eqref{eq:nonneg} is satisfied. 
% The challenge however is in satisfying the property in Equation~\eqref{eq:norm}.
% In fact, we can show that the standalone functional form \eqref{eq:fsos} is actually fundamentally unable to satisfy \eqref{eq:norm} for non-trivial cases.
The difficulty lies in ensuring that the  condition in Equation~\eqref{eq:norm} also holds. 
In fact, in Section~\ref{sec:degen}, we study the inherent restrictions of \eqref{eq:fsos}, when satisfying the normalization condition in \eqref{eq:norm}.

\subsection{Conditional Normalization}
The models proposed in this work revolve around the ability to perform exact analytical integrals of the density. Analytical integration is crucial to formulating feasible normalization constraints, as well as analytical belief propagation.
To study the normalization criterion in Equation~\eqref{eq:norm}, we must first formally define an algebra of basis functions which possess closed-form anti-derivatives.
\begin{definition} [Integrable Multiplicative-Algebra]
    An \textit{Integrable Multiplicative-Algebra} is a class of functions $\mathcal{A} \subset \fspace(\uspace^d)$ such that (i) for any $b_i, b_j \in \mathcal{A}$, 
    \begin{equation}
        b_i(\px) b_j(\px) \in \mathcal{A} \qquad \text{ (multiplicative closure),}
    \end{equation}
    and (ii) for every $b(\px) \in \mathcal{A}$, its antiderivative $B(\px)$ is known in closed form, i.e., 
    \begin{equation}
        \frac{\partial}{\partial \px_1} \cdots \frac{\partial}{\partial \px_d} B(\px) = b(\px).
    \end{equation}
\end{definition}
There are many known integrable multiplicative-algebras. A non-exhaustive list is: polynomials with integer or real exponents, trigonometric functions with linear arguments, exponential functions with linear arguments, Gaussian kernels, Beta-PDFs, etc.

Therefore, restricting the family of basis functions to the integrable multiplicative-algebra facilitates exact analytical integration.
%We study basis functions that belong to an integrable multiplicative-algebra $\mathcal{A}$ to permit exact integration. 
Additionally, we make one more structural assumption on the basis functions to allow us to understand functional properties while remaining general to arbitrary choices of basis function families.

\begin{assumption} \label{assumption}
    The basis functions are separable in the dependent variable $\px'$ and conditioner variable $\px$. Namely, $\pb(\px, \px') = \pphi(\px) \circ \ppsi(\px')$, for some arbitrary $\pphi$ and $\ppsi$
    such that $b_i(\px, \px') = \phi_i(\px) \psi_i(\px')$. 
\end{assumption}
At first glance, it may seem that Assumption \ref{assumption} is restrictive; however, monomial basis functions (a common choice for \sos{} expressive functional optimization) are separable in all variables. 
Moreover, Assumption \ref{assumption} does \textit{not} require $\pphi$ and $\ppsi$ to be separable in the components of the state, e.g., $\phi(x_1, x_2)$ 
does not
need to be equal to $\phi_1(x_1) \phi_2(x_2)$.
%Additionally, we make no assumptions about the separability between components of the random variables, e.g. between $x_i$ and $x_j$. \qh{this paragraph seems important but pretty unclear}

To formulate the normalization criterion in Equation~\eqref{eq:norm}, we can rewrite Equation~\eqref{eq:fsos} in a double-sum representation
\begin{align}
    \int_{\uspace^d} f(\px, &\px') d\px'  \notag \\
    &=\int_{\uspace^d} \sum_{i=1}^n \sum_{j=1}^n q_{i,j} \phi_i(\px)\psi_i(\px') \phi_j(\px) \psi_j(\px') d\px'  \notag \\
    &=\sum_{i=1}^n \sum_{j=1}^n q_{i,j} \phi_i(\px)\phi_j(\px) \finnerprod{\psi_i}{\psi_j} \label{eq:dub_sum_int}
\end{align}
Let $\Gamma$ be the Gram matrix of $\boldsymbol{\psi}(\cdot)$ where each element $\gamma_{i,j} = \finnerprod{\psi_i}{\psi_j}$. The following property holds.

\begin{lemma} \label{lem:int_y}
    If $f(\px, \px')$ is a \sos{} form then $\int_{\uspace^d} f(\px, \px') d\px'$ is \textit{also} a \sos{} form.
\end{lemma}
\begin{proof}
    $\Gamma$ is a Gram matrix which is guaranteed to be PSD. Substituting $\gamma_{i, j}$, \eqref{eq:dub_sum_int} can be re-expressed as $\pphi^T(\px) \big(\Gamma \circ \pQ \big) \pphi(\px)$. By the Schur product theorem, $\Gamma \circ \pQ$ is PSD. 
\end{proof}

%\subsection{Existence of \sos{}-form CDE \qh{Degeneracy of \sos{}-form CDE?}}
\subsection{Restrictions of \sos{}-form CDEs} \label{sec:degen}
Lemma \ref{lem:int_y} describes a key property of the proposed \sos{} form: integrating out the dependent variable induces another \sos{} form in just the $\px$-basis. This section highlights a counter-intuitive implication of this result. Despite having rich functional approximation properties, \edt{the \sos{} form in \eqref{eq:fsos} requires a highly restrictive structure when modeling a valid conditional distribution.}
%can \textit{never} model a non-trivial CDE with linearly independent $\px-\px'$ separable basis functions.

Accounting for the symmetry in $\pQ$, 
Equation~\eqref{eq:norm} can be rewritten as
%\eqref{eq:dub_sum_int} can be re-written as 
\begin{equation} \label{eq:dub_sum_unique_terms}
    \sum_{i=1}^n q_{i,i} \gamma_{i, i} \phi_i^2(\px) + \sum_{i=1}^n \sum_{j=i+1}^n 2 q_{i, j} \gamma_{i, j} \phi_i(\px) \phi_j(\px) = 1.
\end{equation}
Let $\mathbf{B}(\px) = \pb(\px) \pb^T(\px)$ and $\pb_\triangle(\px)$ denote the vector of all upper triangular elements of $\mathbf{B}(\px)$.

\begin{lemma} \label{lem:lin_ind}
    For Equation~\eqref{eq:dub_sum_unique_terms} to hold, either (i) $q_{i, j} = 0$ for all non-constant basis function products $\phi_i(\px) \phi_j(\px)$, or (ii) there is linear dependence between at least two $\phi_i(\px)\phi_j(\px)$ and $\phi_k(\px)\phi_l(\px)$.
\end{lemma}

\begin{proof}
    Suppose all $\phi_i(\px) \phi_j(\px)$ are linearly independent for $1 \leq i < j \leq n$, i.e., for a vector $\mathbf{a} \in \reals^{n(n-1)/2}$,
    \begin{align} \label{eq:lin_ind_lem}
        &\mathbf{a}^T \pb_\triangle(\px) = 0 \; \forall \px \in \uspace^d  \iff \mathbf{a} = \mathbf{0}.
    \end{align}
    Following from \eqref{eq:lin_ind_lem}, the partial derivatives satisfy
    \begin{align}
        &\frac{\partial}{\partial x_i} \mathbf{a}^T \pb_\triangle(\px) = 0 \; \forall \px \in \uspace^d \; \forall i \in \{1, \ldots, n\} \label{eq:lin_ind_partials}
    \end{align}
    holds only if $a_j = 0$ for all non-constant elements of $\pb_\triangle(\px)$.
    Taking the partial derivatives of the normalization criterion in \eqref{eq:dub_sum_unique_terms} yields the same equations as in \eqref{eq:lin_ind_partials}, and thus $q_{i, j} = 0$ for all non-constant $\phi_i(\px)\phi_j(\px)$.
\end{proof}
%\begin{proof}
%    Suppose all $\phi_i(\px) \phi_j(\px)$ are linearly independent for $1 \leq i < j \leq n$, i.e., for a vector $\mathbf{a} \in \reals^{n(n-1)/2}$,
%    \begin{align} \label{eq:lin_ind}
%        &\mathbf{a}^T \pb_\triangle(\px) = 1 \; \forall \px \in \uspace^d  \iff \mathbf{a} = \mathbf{0}.
%    \end{align}
%    Assuming \eqref{eq:lin_ind}, the partial derivatives satisfy
%    \begin{align}
%        &\frac{\partial}{\partial x_i} \mathbf{a}^T \pb_\triangle(\px) = 0 \; \forall \px \in \uspace^d \; \forall i \in \{1, \ldots, n\}
%    \end{align}
%    holds only if $a_j = 0$ for all non-constant elements of $\pb_\triangle(\px)$.
%\end{proof}

Lemma \ref{lem:lin_ind} clarifies the necessity of linear dependence between products of basis functions. Intuitively, all non-constant functions appearing in \eqref{eq:dub_sum_unique_terms} must sum to zero, leaving only a constant term. More generally, the restrictions on $\pphi$ are as follows.

%Lemma \ref{lem:lin_ind} shows that 
%Let $span^2(\pb)$ be the set of all bi-linear combinations of basis functions, i.e. 
%\begin{equation}
%    span^2(\pb) = \{f \in \fspace(\uspace^d) \text{ s.t. } f = \pb(\px) \pA \pb(\px), \pA \in \reals^{n\times n} \}.
%\end{equation}

%\begin{lemma}
%    Let $\pb(\px) = [\pb_1(\px), \ldots, \pb_n(\px)]$, if every $\pb^2_{i}(\px)$ is linearly dependent with $\pb_l(\px) \pb_j(\px)$ for some $j\neq l$ for all $i \in \{1, \ldots, n\}$, then all $\pb_i(\px)$ are constant on $\uspace^d$.
%    If all for all $i\in \{1, \ldots, n\}$, 
%    \begin{equation}
%        \pb^2_{i}(\px) = 
%    \end{equation}
%    \ml{complete this.}
%\end{lemma}
%\begin{proof}
%    \ml{move proof to the appendix}
%\end{proof}

\begin{theorem} \label{thm:nonexist}
    Let $\pb(\px, \px')$ be a set of linearly independent, $\px-\px'$ separable basis functions. Then  
    $f(\px, \px') = \pb^T(\px, \px') \; \pQ \; \pb(\px, \px')$, with $\pQ \psd$, satisfies condition in Equation~\eqref{eq:norm} if and only if \edt{
    \begin{equation} \label{eq:ellipse}
        \| \pxi(\px) \| = 1
        %\| \mathbf{M} \pphi(\px) \| = 1
    \end{equation}
    where $\pxi(\px) = \big(\Gamma \circ \pQ \big)^{1/2} \pphi(\px)$ and $(\cdot)^{1/2}$ denotes a matrix square root.
    } \rmv{$f(\px, \px') = f(\px')$, i.e., $f$ is not dependent on $\px$.}
\end{theorem}
\begin{proof}
    Following from Lemma \ref{lem:int_y}, $\Gamma \circ \pQ$ is PSD, and therefore can always be decomposed using a matrix square root as $\Gamma \circ \pQ = \mathbf{M}^T \mathbf{M}$.
    The LHS of \eqref{eq:norm} can then be written as 
    \begin{align}
        \pphi^T(\px) \big(\Gamma \circ \pQ \big) \pphi(\px)  
        = &\Big(\pphi^T(\px) \mathbf{M}^T \Big) \Big( \mathbf{M} \pphi(\px) \Big) \notag \\
        = &\|\pxi(\px) \|^2 
    \end{align}
    Constraint \eqref{eq:norm} is therefore equivalent to \eqref{eq:ellipse}.
\end{proof}

%\begin{proof}
%    Recall, from the proof of Lemma \ref{lem:fixed_point}, Eqaution~\eqref{eq:norm} can be expressed as 
%    \begin{equation} \label{eq:pf1}
%        \pphi^T(\px) \big(\Gamma \circ \pQ \big) \pphi(\px) = 1 \qquad \forall \px \in \uspace^d, 
%    \end{equation}
%    where $\Gamma \circ \pQ = \bar{\pQ} \psd$. 
%    %Suppose that $\phi_i(\px) \phi_j(\px)$ is linearly independent from $\phi_k(\px) \phi_l(\px)$ for all $(i, j) \neq (k, l)$. 
%    Then, \eqref{eq:pf1} holds iff i) there exists a constant basis function $\phi_c(\px) = c$, $c \in \reals$ with corresponding coefficient $q_c$ such that $q_c = c^{-1}$, and ii) $\bar{q}_{i,j} = 0$ for all $\langle \psi_i, \psi_j \rangle \neq 0$. For any (non-degenerate) $\psi_i$, $\finnerprod{\psi_i}{\psi_i} = \| \psi_i \|^2 \geq 0$. Additionally, if $q_{i,i} = 0$ then $q_{i, j} = q_{j, i}= 0$ $\forall j \in \{1, \ldots, n\}$, removing all terms involving $\phi_i$ from $f$, effectively removing $\phi_i(\px) \psi_i(\px')$ from the basis set. Thus, $\bar{q}_{i, i} \geq 0$.
%\end{proof}

%The proof of Theorem \ref{thm:nonexist} is provided in the Appendix.

Theorem \ref{thm:nonexist} illuminates a critically limiting property of \sos{} forms for conditional density estimation. 
Specifically, $\pphi(\px)$ is restricted to functions that parameterize an ellipse manifold. 
Such a $\pphi$ can be constructed by normalizing the output, i.e.,
\begin{equation} \label{eq:unit_vector}
    \pxi(\px) = \frac{\tilde{\pxi}(\px)}{\| \tilde{\pxi}(\px) \|}
\end{equation}
where $\tilde{\pxi}$ is some arbitrary basis. However, constructing $\pphi$ via \eqref{eq:unit_vector} generally necessitates non-constant functions in the denominator, and therefore forfeits analytical integrability, i.e. $\pxi \not \in \mathcal{A}$ even if $\tilde{\pxi} \in \mathcal{A}$. There exist special choices of $\pphi$ that automatically satisfy \eqref{eq:ellipse} without normalization, namely, when $\pxi(\px) \circ \pxi(\px)$ is a partition of unity (e.g., when $\pxi(\px)$ is a Bernstein polynomial basis as in \cite{amorese2026universal}). However, such special structures are often rigid and do not allow for sparse representations and/or optimization of the parameters of the basis functions.
%To illustrate this limitation, consider two means of constructing such a $\pphi$: (i) normalization and (ii) bounded partition-of-unity.

%Enforcing \eqref{eq:ellipse} while also ensuring $\pphi(\px) \in \mathcal{A}$ is very challenging, generally requiring a known partition-of-unity structure or a bounding relaxation procedure.

To overcome this critical restriction, we present a rational \sos{} form for conditional density estimation that employs the flexible and expressive \sos{} form, while bypassing the aforementioned restrictions of the normalization constraint.

\section{RATIONAL FACTOR FORM}
We propose the following \textit{rational-factor} (RF) form for modeling each piece of the Markov Process by introducing a factor function $g(\cdot)$:
\begin{align}
    p(\px' | \px) &= \frac{g(\px') \tf(\px, \px')}{g(\px)} \label{eq:rff_cond} \\
    p(\px_0) &= g(\px_0) h_0(\px_0) \label{eq:rff_init}
\end{align}
where $\tf$, $g$, and $h_0$ are all \sos{} form functionals. By ensuring $g(\cdot) > 0$, equations~\eqref{eq:rff_cond} and \eqref{eq:rff_init} are non-negative and well defined. Let $g(\cdot) = \pphi^T_g(\cdot) \pR \pphi_g(\cdot)$. To guarantee that $g$ is strictly positive, it is sufficient to ensure that $\pR \posdef$ and $\phi_g(\cdot) \neq 0$ on $\uspace^d$ for at least one $\phi_g(\cdot) \in \pphi_g(\cdot)$. 

Crucially, the RF form maintains the analytical feasibility of belief propagation. This can be seen by using \eqref{eq:marg} to compute $p(\px_1)$ assuming the RF form:
\begin{align}
    p(\px_1) &= \int_{\px_0} p(\px_1 | \px_0) p(\px_0) d\px_0 \notag \\
    &= \int_{\px_0} \frac{g(\px_1) \tf(\px_{0}, \px_1)}{g(\px_{0})} \big(g(\px_0) h_0(\px_0) \big) d\px_0 \notag \\
    &= g(\px_1) \int_{\px_0} \tf(\px_{0}, \px_1) h_0(\px_0) d\px_0 \label{eq:rff_prop} \\
    &= g(\px_1) h_1(\px_1). \label{eq:rff_h1}
\end{align}
with $h_1(\px_1) := \int_{\px_0} \tf(\px_{0}, \px_1) h_0(\px_0) d\px_0$.
Since $\tf(\px_0, \px_1) \in \mathcal{A}$ and $h_0(\px_0) \in \mathcal{A}$, $\tf(\px_0, \px_1) h_0(\px_0) \in \mathcal{A}$ and thus \edt{the} integral in \eqref{eq:rff_prop} is analytically feasible. While $h_1$ may not necessarily be of \sos{} form, it is guaranteed to be in $\mathcal{A}$. Thus, one can recursively repeat the above procedure to compute $p(\px_k)$ for any $k > 0$.

\subsection{Formulating the Conditional Normalization Constraint} \label{sec:form_norm}
The RF form naturally permits non-negativity, and analytical belief propagation, leaving only the formulation of the normalization constraint.
Intuitively, $g(\px)$ is responsible for normalizing $\tf$ at each conditioner $\px_{k-1}$. As a consequence, the expressions of $\tf$ and $h_0$ must adjust their density expressions to compensate for the factor $g(\px')$. Assuming RF form, the constraint in Equation~\eqref{eq:norm} can be rewritten as
\begin{align}
    \int_{\px'} \frac{g(\px') \tf(\px, \px')}{g(\px)} d\px' &= 1 \notag \\
    \implies \int_{\px'} g(\px') \tf(\px, \px') d\px' &= g(\px).
\end{align}
Using \sos{} forms,
\begin{align}
    %\int_{\px_{k}} \Big(\phi^T_g(\px_k) \pR \phi_g(\px_k)\Big)\Big( \phi^T_f(\px_{k-1}) \circ \psi^T_f(\px_k) Q \phi_f(\px_{k-1}) \circ \psi_f(\px_k) \Big)  %\tf(\px_{k-1}, \px_k) d\px_{k} &= g(\px_{k-1}).
    \int_{\px'} \Big(\pphi^T_g(\px') \pR \pphi_g(\px')\Big)\Big( \pb_{\tf}(\px,& \px')^T \pQ \pb_{\tf}(\px, \px') \Big) d\px' \notag \\
    &= \pphi^T_g(\px) \pR \pphi_g(\px) \label{eq:rff_sos}
\end{align}
where $\pb_{\tf}(\px, \px') = \pphi_{\tf}(\px) \circ \ppsi_{\tf}(\px')$. Integrating over $\px'$ in the LHS of \eqref{eq:rff_sos} yields a linear combination of products of $\pphi_{\tf}(\px)$, whereas the RHS is a linear combination of products of $\pphi_g(\px)$. Thus, equality (for non-trivial cases) requires $g$ and $\tf$ to span the same function space. This can be enforced simply by letting $g$ and $\tf$ share the same $\px$-basis functions, i.e., $\pphi_{\tf}(\px) = \pphi_g(\px)$, which we simply denote as $\pphi(\px)$.
Normalization can then be formulated via equality between coefficients for each respective product function $\phi_i(\px)\phi_j(\px)$. 

For ease of presentation, we denote $\pv^\pR, \pv^\pQ \in \reals^{n^2}$ as the vectorization of the elements of $\pR$ and $\pQ$ respectively. Each element $v^\pR_{(i,j)}$ of $\pv^\pR$ corresponds to a product of basis functions $\phi_i(\px)\phi_j(\px)$. Then, the normalization constraint can be expressed by a simple linear relationship.

\begin{lemma} \label{lem:fixed_point}
    The following linear relationship is sufficient for satisfying the normalization constraint \eqref{eq:rff_sos}
    \begin{equation} \label{eq:fixed_point}
        \pv^\pR = \diag(\pv^\pQ) \pE \pv^\pR,
    \end{equation}
    where $\pv^\pR$ and $\pv^\pQ$ are vectors containing the elements in $\pR$ and $\pQ$ respectively and $\pE \in \reals^{n^2 \times n^2}$ is a matrix that depends only on $\pphi(\cdot)$ and $\ppsi_{\tf}(\cdot)$ with elements
    \begin{equation} \label{eq:E}
        e_{(i,j),(k,l)} = \finnerprod{\phi_i(\px') \phi_j(\px')}{\psi_k(\px') \psi_l(\px')}.
    \end{equation}
\end{lemma}

The proof is provided in the Appendix.
Lemma \ref{lem:fixed_point} illuminates the circularity of using the same function $g(\cdot)$ in the denominator (with argument $\px$) and as a factor in the numerator (with argument $\px'$). With $\pQ \psd$, ensuring the elements of $\pR$ satisfy the condition in Equation~\eqref{eq:fixed_point} and $\pR \posdef$ yields a valid CDE. 

Once a valid RF form CDE is obtained, and thus $g(\cdot)$ is known, learning the initial belief is straightforward. Since $g(\cdot) > 0$, density estimation of $p(\px_0)$ can be equivalently recast as using $h_0(\px_0)$ to approximate $p(\px_0)(g(\px_0))^{-1}$. Since $g(\px_0)h_0(\px_0) \in \mathcal{A}$ and is analytically integrable, the normalization constant can be obtained exactly. 

\begin{remark}
    The RF form is not restricted to \sos{} functionals. By substituting non-negative linear combinations of non-negative basis functions for $\tf$, $g$, and $h_0$, the estimators can inherit the same properties as \sos{} forms. However, doing so greatly restricts the choice of $\mathcal{A}$, as many universal approximators (e.g., monomial-basis polynomials) lose their universality when restricted to non-negative coefficients.
\end{remark}

%We have shown that the RF form easily permits non-negativity and analytical belief propagation, hence, guaranteeing the conditional normalization constraint remains.

\subsection{Time and Memory Complexity}
Let $n_\theta$ be the number of parameters of a basis function in $\pphi$ or $\ppsi$. Then the model can be stored with $\mathcal{O}(n^2 + n_\theta)$ parameters. Belief propagation has time complexity $\mathcal{O}(n_\theta n^4)$ dominated by the product of two \sos{} forms in Equations~\eqref{eq:rff_prop} and \eqref{eq:E}. While $n_\theta$ is often implicitly dependent on $d$, for many typical basis functions (e.g., monomials), this dependence is linear. 
%With a direct parameterization iteration of has time complexity of $\mathcal{O}(n_\theta n^6)$ dominated by the eigen-decomposition of $\pE$ expansion of the 

%\subsection{Extension to full Bayesian Networks}
\begin{remark}
    The RF form can be extended to model Bayesian Networks (and time-inhomogenous Markov processes), e.g., $p(\pz, \py, \px)$, with non-Markovian dependency, e.g.
    \begin{equation}
        p(\pz, \py, \px) = p(\pz | \py, \px) p(\py | \px) p(\px).
    \end{equation}
    In fact, since, for such models, the conditional densities are not the same and do not require the same recursion as a time-homogeneous Markov process. Thus, the factor functions $g$ in the numerator and denominator need not be the same, making the normalization constraint in Equation~\eqref{eq:fixed_point} a simple matrix equality (rather than an eigenvalue problem).
\end{remark}

\section{CONSTRAINED TRAINING}
In this section, we discuss the necessary pieces for formulating a tractable constrained optimization problem for training the Markov Process. 
Recall, for $p(\px' | \px)$ to be a valid CDE, both $\pQ$ and $\pR$ must belong to the PSD and PD cones respectively, and are subject to the equality constraints in Equation~\eqref{eq:fixed_point}. Fortunately, the constraints mirror those in a SDP \citep{vandenberghe1996semidefinite}. Therefore, we can leverage SDP techniques to enforce the constraints during training.

The learning problem, however, cannot be formulated as an SDP and is non-convex for two reasons. Firstly, common density estimation objectives, such as Expected Log Likelihood (LLH), do not admit a linear objective function required by SDP. Secondly, the proposed functional form permits \textit{optimization of the basis function parameters} themselves. This allows the optimizer to reduce redundancy between basis functions, and thereby achieving a sparse representation, at the cost of making the problem potentially highly non-convex depending on the choice of basis functions. Due to the non-convexity and the potentially large supply of data, we use stochastic gradient-descent (SGD) methods to train both $p(\px' | \px)$ and $p(\px_0)$.

\subsection{Enforcing the Conditional Normalization Constraint}
%The CDE is valid if $\pQ \psd$, $\pR \posdef$, and \eqref{eq:fixed_point} is satisfied.
%Firstly, enforcing $\pQ \psd$ can be done via a direct parmeterization of an unconstrained matrix $\pL$ such that $\pQ = \pL^T \pL$.
Let $\pQ = \pL^T \pL$ for some unconstrained parameter matrix $\pL$, automatically making $\pQ \psd$.
Since the parameters of the basis functions are subject to change during training, the equality constraint in Equation~\eqref{eq:fixed_point} may change as well, since $\pE$ is dependent on the basis functions. Furthermore, equality constraints in SDP are notoriously challenging, requiring specialized techniques, e.g. augmented Lagrangians \citep{burer2003nonlinear}. In practice, if the technique has not properly converged, there may be non-negligible numerical errors in the normalization of the CDE. Consequently, when predicting beliefs in the future (using Equation~\eqref{eq:rff_prop}), the numerical errors can accumulate leading to substantial errors in the predictions.

To mitigate this problem, it is highly preferable to avoid iterative equality constraint techniques and rather use a direct parameterization of $\pR$ such that \eqref{eq:fixed_point} holds exactly (up to floating-point precision). 
Rearranging Equation~\eqref{eq:fixed_point} as
\begin{equation} \label{eq:eigenvalue}
    \big(\diag(\pv^\pQ) \pE - I) \pv^\pR = \mathbf{0}
\end{equation}
shows that a feasible $\pv^\pR$ is the eigenvector of the matrix $\pM = \diag(\pv^\pQ) \pE$ corresponding to eigenvalue $\lambda^\pM = 1$. 

Obtaining a feasible $\pv^\pR$ can be achieved in two steps: (i) scaling $\pv^\pQ$ such that $\pM$ has at least one real eigenvalue $\lambda^\pM = 1$, and (ii) computing the corresponding eigenvector. For (i), we can simply scale $\pQ$ by $(\lambda^\pM_{>0})^{-1}$ where $\lambda^\pM_{>0}$ is any positive real eigenvalue of $\pM$. If $\pM$ does not have at least one positive real eigenvalue, a feasible $\pR$ can not easily be obtained without manipulating the basis functions themselves, and thus a loss penalty can be applied. However, we found that this rarely happens in practice. By rescaling $\pQ$ by $(\lambda^\pM_{>0})^{-1}$, $\pM$ is also effectively scaled by the same quantity, thereby guaranteeing an eigenvalue of $1$. The corresponding feasible eigen vector $\pv^\pR$ can then be found via an eigen decomposition.

To enforce that $\pR \posdef$, we can use a log-determinant barrier, as is common in interior-point SDP solvers. Specifically, $\log \det(\pR) \rightarrow \infty$ as any eigenvalues approach $0$, i.e., the semi-definite boundary.
SGD is unpredictable and may exit the feasible region during training, thus a loss penalty on the non-positive eigenvalues can be applied to guide the parameters into the PSD cone, captured by the following loss term
\begin{align} \label{eq:ldb}
    &\mathcal{L}_{\pR} =  \notag\\
    &\begin{cases}
        \log \det(\pR) & \text{ if all } \lambda^\pR \geq 0\\
        \sum_{i=1}^n \Big(\max(-\text{Re}(\lambda^\pR_i) 
        + \epsilon, 0)  \\ \notag \qquad + \text{Im}(\lambda + \epsilon)\Big)^2 & \text{otherwise}
    \end{cases}
\end{align}
for some $\epsilon > 0$.

The cumulative loss function for training the RF form CDE is
\begin{equation}
    \mathcal{L} = -\expect_{\px', \px}\big[\log p(\px' | \px)\big] - \mathcal{L}_{\pR},
\end{equation}
where the first term is the empirical negative log-likelihood of the model. Additional regularization loss can be added depending on the chosen family of basis functions.

\section{EXPERIMENTS}
The aforementioned sections layout the theoretical formulations of the RF form Markov model. However, the utility of the proposed model is contingent on its ability to learn realistic high-dimensional systems. This section validates the theoretical prowess of the RF form Markov model, achieving exact (with respect to the learned model) belief propagation for systems of up to 12 dimensions.
Details of the experiment setups and system dynamics along with more elaborate discussions on the results are provided in the Appendix.

\subsection{Choice of basis functions}
The experiments were performed using component-wise products of 1D Beta-pdf $\pphi$ (and similarly $\ppsi$) basis functions of the form
\begin{equation} \label{eq:beta_pdf}
    \phi_i(\px) = \prod_{m=1}^d x_m^{1-\alpha_{i, m}} (1-x_m)^{1-\beta_{i, m}} 
\end{equation}
where $\alpha_{i, m}$ and $\beta_{i, m}$ are trained parameters. 
This family of basis functions was chosen for direct comparison with Bernstein Normalzing Flow (BNF), since Beta-PDFs are closely related to the Bernstein basis polynomials.
%This family of basis functions was chosen for the smooth ``bell-shape'', inspired by the efficacy of Beta-mixture distributions for modeling distributions over $\uspace^d$ \citep{rousseau2010rates}.

\begin{figure}[h]
    \centering
    \includegraphics[width=0.98\linewidth, trim=00 00 45 45, clip]{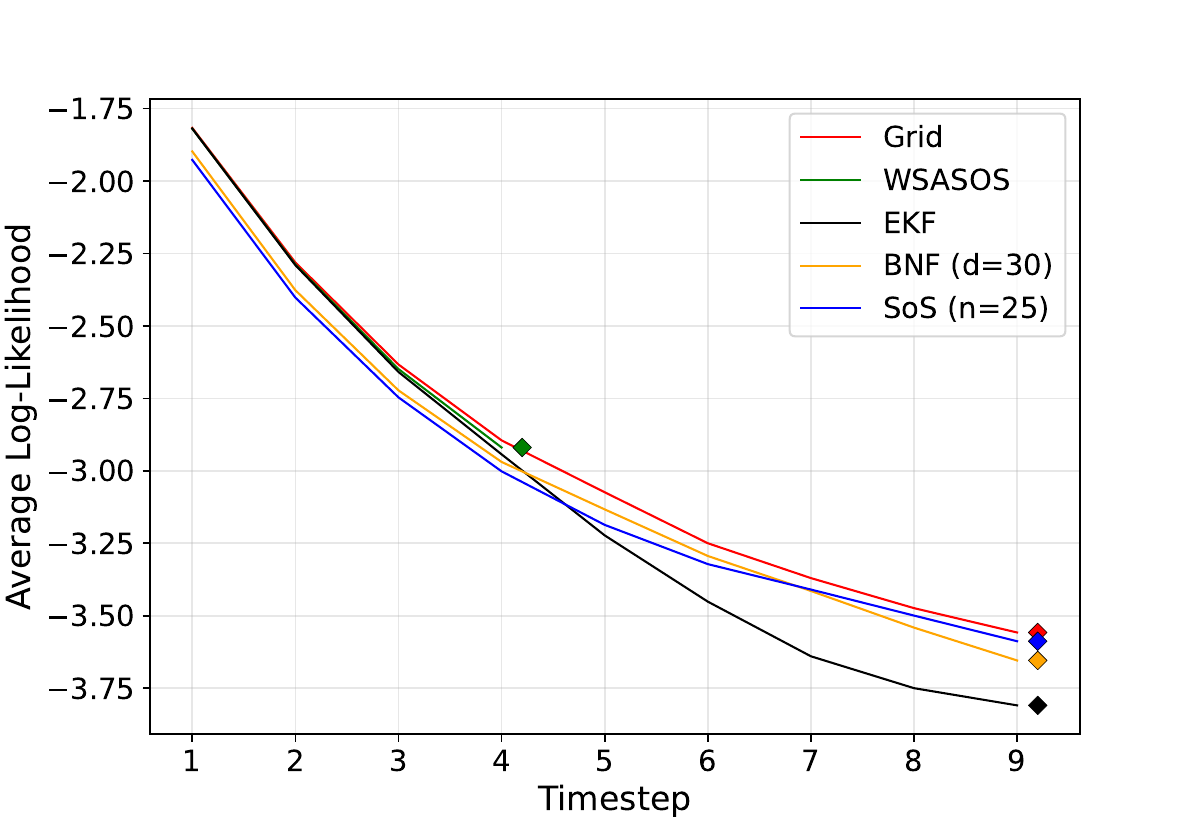}
    \caption{Accuracy Comparison on 2D system against state-of-the-art density propagation methods}
    \label{fig:comparison}
\end{figure}
\begin{figure*}[ht]
    \centering
    \begin{minipage}[c]{0.05\textwidth}
        \centering
        \rotatebox{90}{$p(x, z)$}
    \end{minipage}%
    % position row
    \begin{minipage}[c]{0.95\textwidth}
    %% pdf pics
    \begin{subfigure}[b]{0.130\linewidth}
        \centering
        \includegraphics[width=\linewidth]{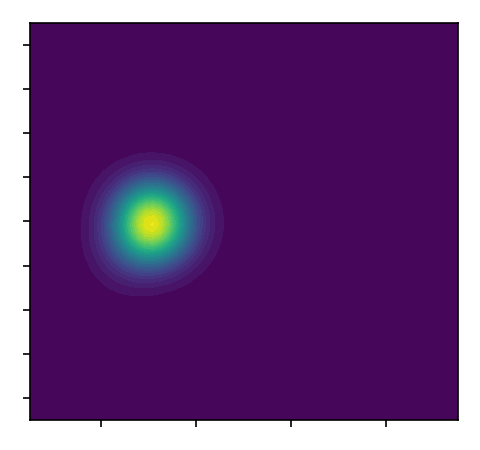}
    \end{subfigure}%
    \begin{subfigure}[b]{0.130\linewidth}
        \centering
        \includegraphics[width=\linewidth]{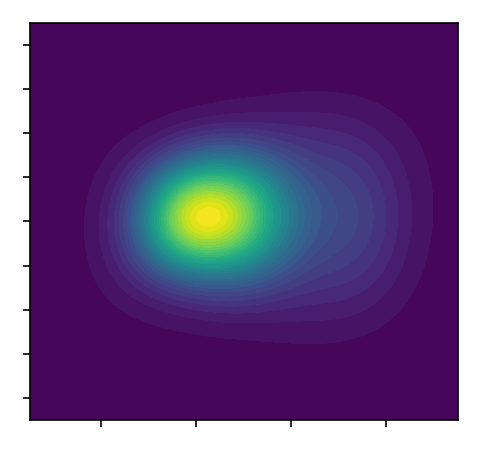}
    \end{subfigure}%
    \begin{subfigure}[b]{0.130\linewidth}
        \centering
        \includegraphics[width=\linewidth]{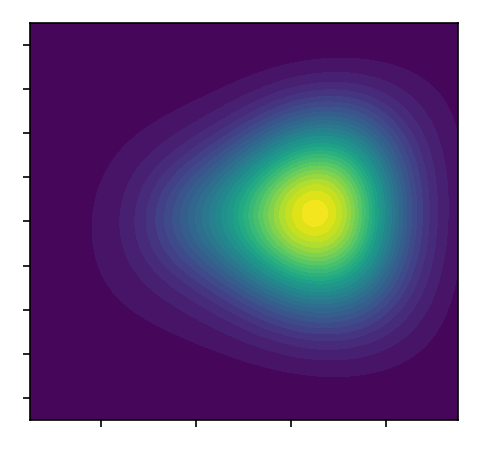}
    \end{subfigure}%
    \begin{subfigure}[b]{0.130\linewidth}
        \centering
        \includegraphics[width=\linewidth]{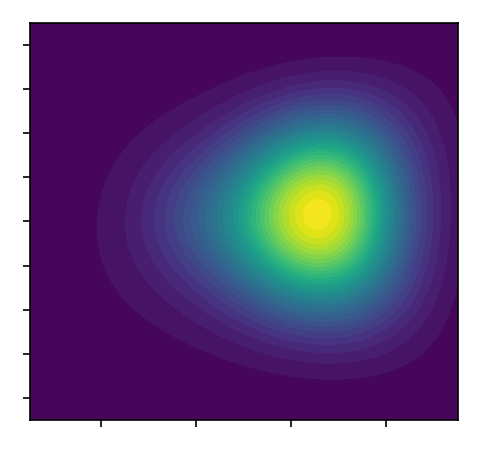}
    \end{subfigure}%
    %% mc pics
    \begin{subfigure}[b]{0.121\linewidth}
        \centering
        \includegraphics[width=\linewidth]{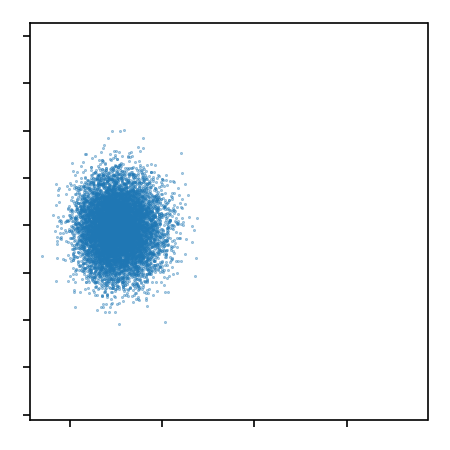}
    \end{subfigure}%
    \begin{subfigure}[b]{0.121\linewidth}
        \centering
        \includegraphics[width=\linewidth]{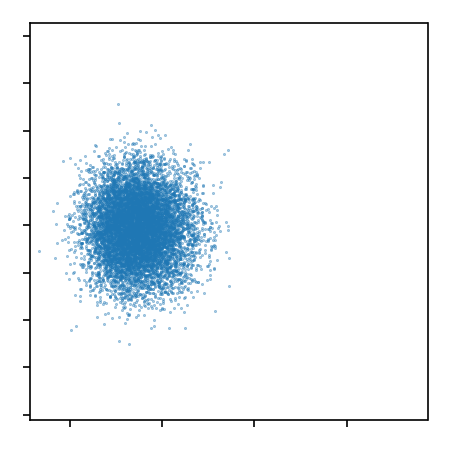}
    \end{subfigure}%
    \begin{subfigure}[b]{0.121\linewidth}
        \centering
        \includegraphics[width=\linewidth]{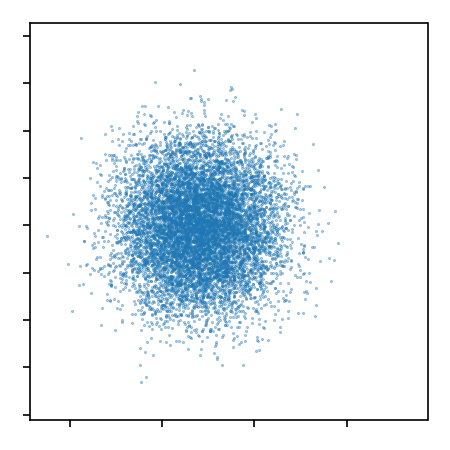}
    \end{subfigure}%
    \begin{subfigure}[b]{0.121\linewidth}
        \centering
        \includegraphics[width=\linewidth]{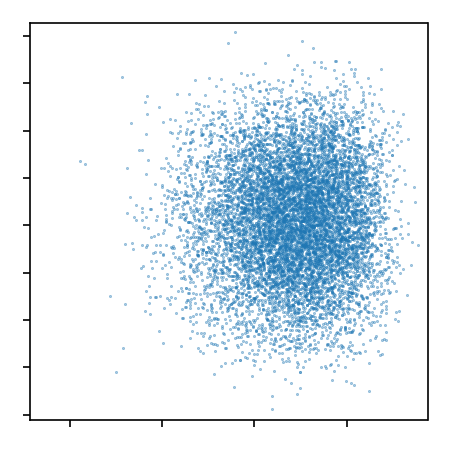}
    \end{subfigure}%
    \end{minipage}
    \centering
    \\
    % angular position row
    \begin{minipage}[c]{0.05\textwidth}
        \centering
        \rotatebox{90}{$p(\theta, v_x)$}
    \end{minipage}%
    %% pdf pics
    \begin{minipage}[c]{0.95\textwidth}
    \begin{subfigure}[b]{0.130\linewidth}
        \centering
        \includegraphics[width=\linewidth]{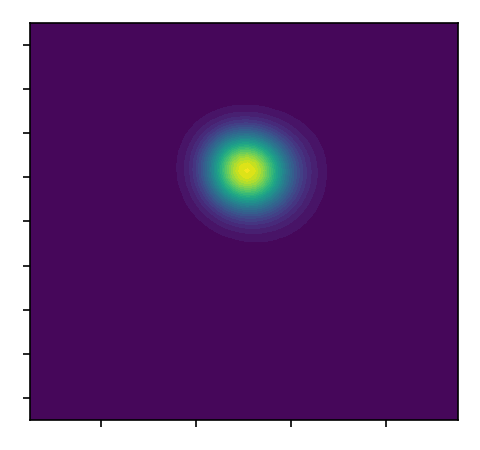}
    \end{subfigure}%
    \begin{subfigure}[b]{0.130\linewidth}
        \centering
        \includegraphics[width=\linewidth]{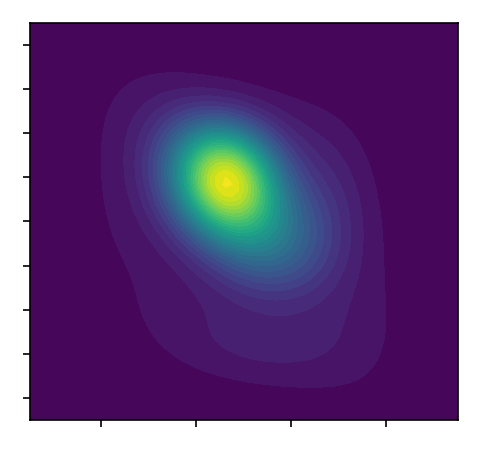}
    \end{subfigure}%
    \begin{subfigure}[b]{0.130\linewidth}
        \centering
        \includegraphics[width=\linewidth]{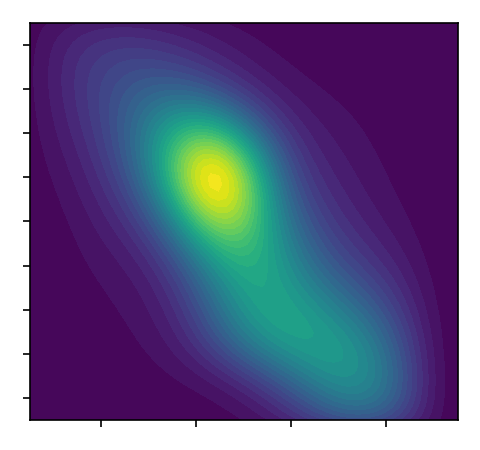}
    \end{subfigure}%
    \begin{subfigure}[b]{0.130\linewidth}
        \centering
        \includegraphics[width=\linewidth]{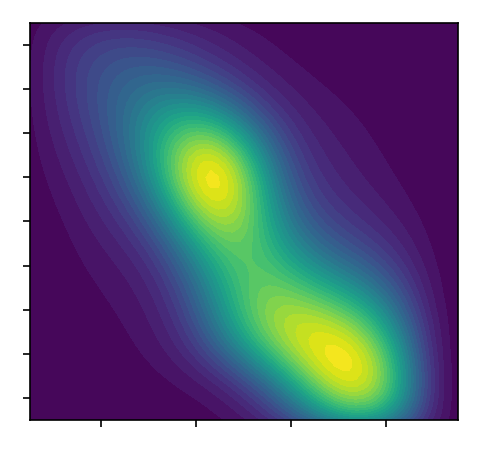}
    \end{subfigure}%
    %% mc pics
    \begin{subfigure}[b]{0.121\linewidth}
        \centering
        \includegraphics[width=\linewidth]{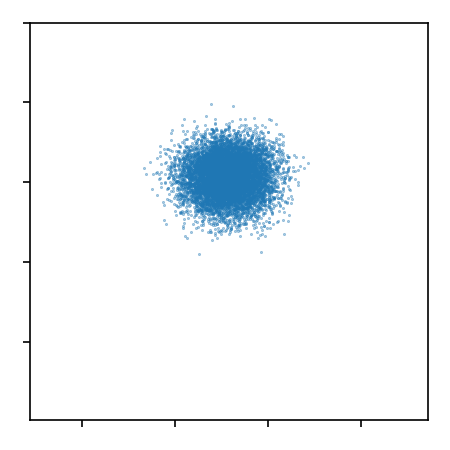}
    \end{subfigure}%
    \begin{subfigure}[b]{0.121\linewidth}
        \centering
        \includegraphics[width=\linewidth]{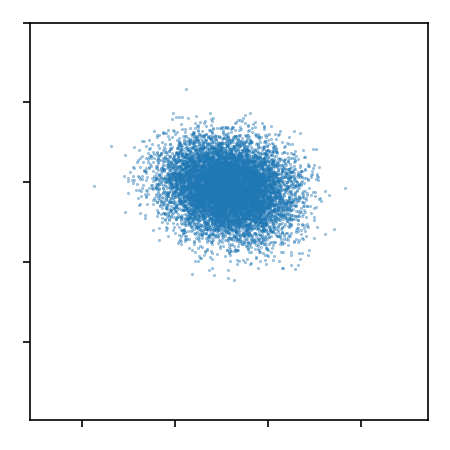}
    \end{subfigure}%
    \begin{subfigure}[b]{0.121\linewidth}
        \centering
        \includegraphics[width=\linewidth]{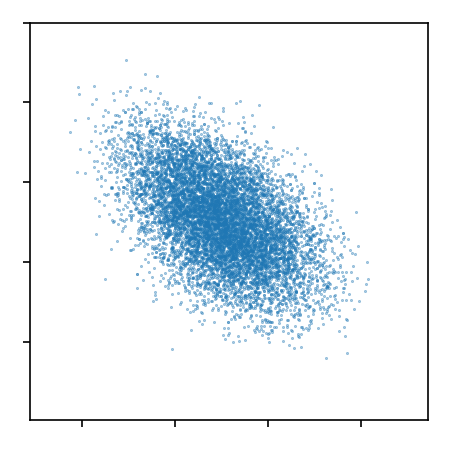}
    \end{subfigure}%
    \begin{subfigure}[b]{0.121\linewidth}
        \centering
        \includegraphics[width=\linewidth]{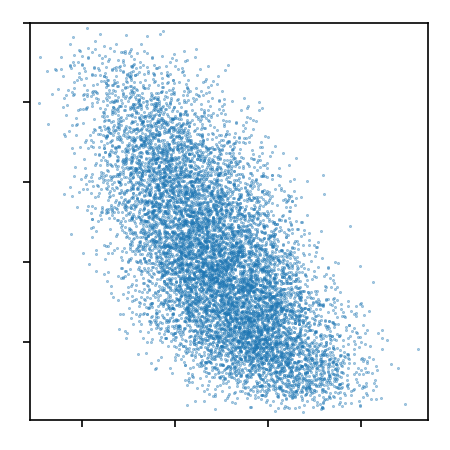}
    \end{subfigure}%
    \end{minipage}
    % velocity row
    \begin{minipage}[c]{0.05\textwidth}
        \centering
        \rotatebox{90}{$p(v_z, \omega)$}
    \end{minipage}%
    \begin{minipage}[c]{0.95\textwidth}
    %% pdf pics
    \begin{subfigure}[b]{0.130\linewidth}
        \centering
        \includegraphics[width=\linewidth]{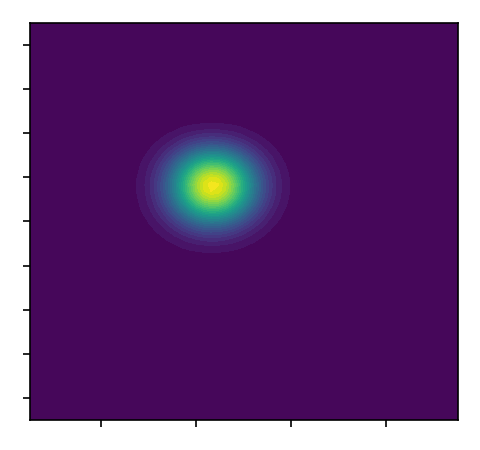}
        \caption*{k = 0}
    \end{subfigure}%
    \begin{subfigure}[b]{0.130\linewidth}
        \centering
        \includegraphics[width=\linewidth]{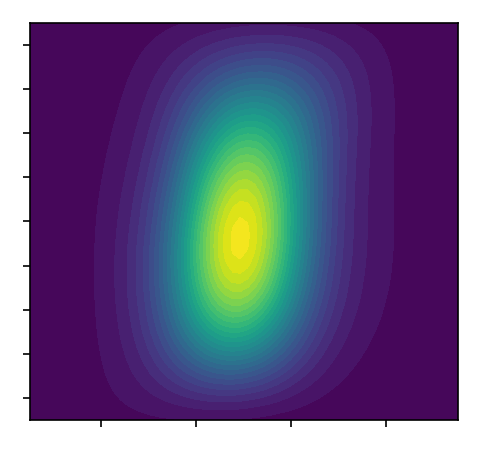}
        \caption*{k = 1}
    \end{subfigure}%
    \begin{subfigure}[b]{0.130\linewidth}
        \centering
        \includegraphics[width=\linewidth]{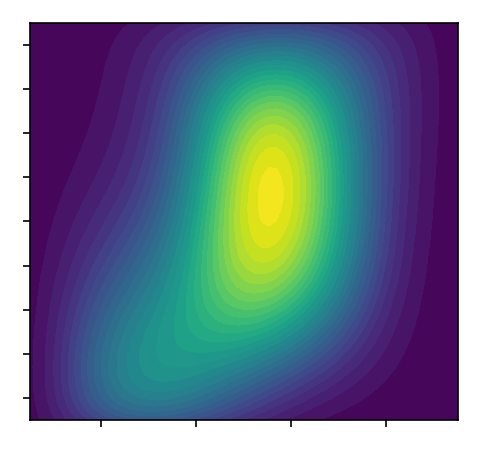}
        \caption*{k = 4}
    \end{subfigure}%
    \begin{subfigure}[b]{0.130\linewidth}
        \centering
        \includegraphics[width=\linewidth]{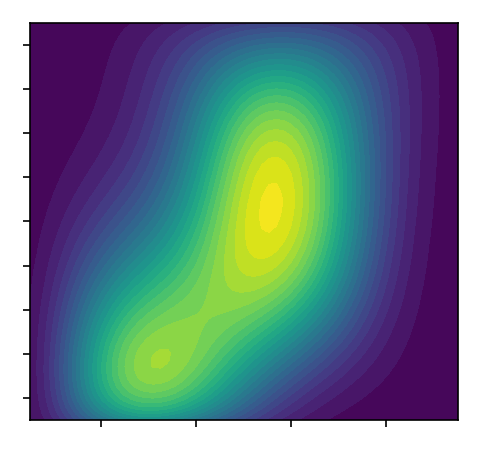}
        \caption*{k = 9}
    \end{subfigure}%
    %% mc pics
    \begin{subfigure}[b]{0.121\linewidth}
        \centering
        \includegraphics[width=\linewidth]{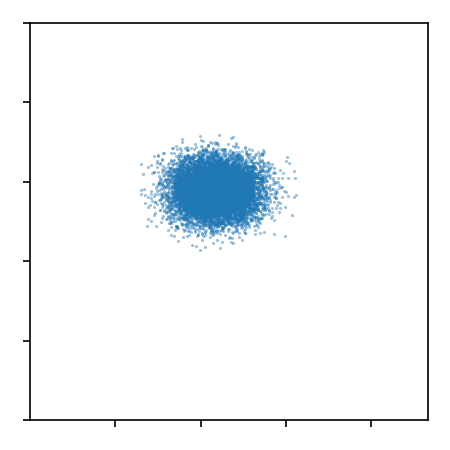}
        \caption*{k = 0}
    \end{subfigure}%
    \begin{subfigure}[b]{0.121\linewidth}
        \centering
        \includegraphics[width=\linewidth]{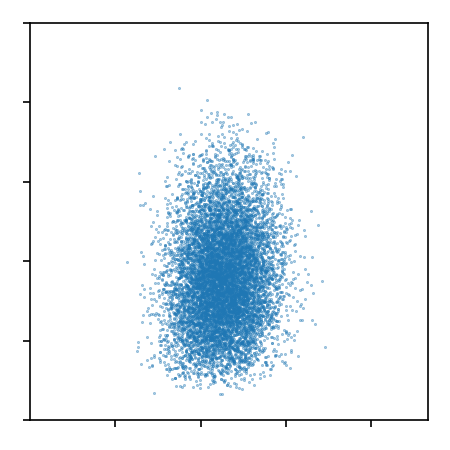}
        \caption*{k = 1}
    \end{subfigure}%
    \begin{subfigure}[b]{0.121\linewidth}
        \centering
        \includegraphics[width=\linewidth]{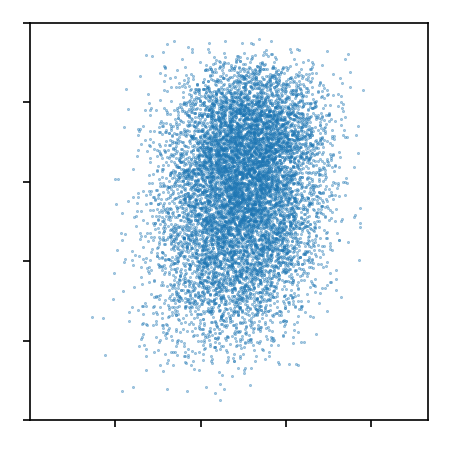}
        \caption*{k = 4}
    \end{subfigure}%
    \begin{subfigure}[b]{0.121\linewidth}
        \centering
        \includegraphics[width=\linewidth]{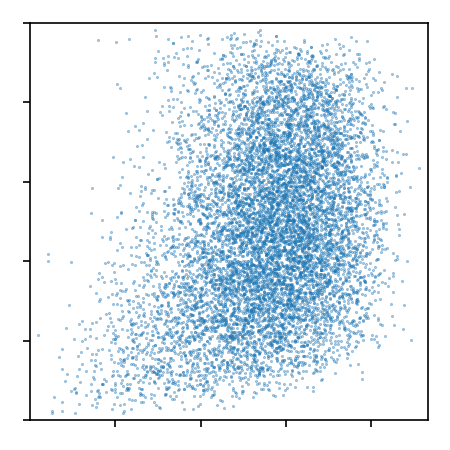}
        \caption*{k = 9}
    \end{subfigure}%
    \end{minipage}
    \caption{Visual Comparison for 6D Quadcopter system. 
    (Left) Learned model. (Right) Monte Carlo simulation (ground truth).
    }
    \label{fig:6d}
\end{figure*}

\subsection{2D Comparison}
To ensure that the \sos{} form does not lose fine-grain accuracy in lower-dimensional problems, we performed an empirical comparison between \sos{}, BNF~\citep{amorese2026universal}, and approximation methods. Specifically we compared against a Gaussian Process regression model using: traditional Extended Kalman Filter (EKF) \citep{schei1997finite} propagation, a grid-based Gaussian Mixture Model (GMM) method \citep{figueiredo2024uncertainty}, and a component-splitting GMM method \citep{kulik2024nonlinearity} (WSASOS). 
Each method was tested on data generated from a 2D Van der Pol system with additive Gaussian noise and evaluated according to the average log-likelihood over Monte Carlo samples generated from the true system. Each experiment was performed 10 times and all log-likelihood were found to have a variance across trials below $10^{-3}$.

The results are shown in Figure~\ref{fig:comparison}.
The grid-based method performs the best, since the underlying system has additive Gaussian noise. This allows the Gaussian Process regression model to learn the true system nearly perfectly. Even with a near-perfect model, EKF looses significant prediction accuracy. The WSASOS method runs out of memory after $k=4$ due to the exponential blow up in the number of components.
\sos{} and BNF perform very comparably with \sos{} having a slight advantage over BNF in the final time steps.

However, the degree $30$ BNF model has over 400k parameters, whereas the $n=25$ \sos{} model has only 1450 parameters. In addition to BNF's inability to scale beyond 2D, the exponential blow-up in parameters evidences strong diminishing returns when using dense Bernstein polynomials. The \sos{} model, however, can achieve the same performance, with orders of magnitude fewer parameters. Belief propagation using \sos{} is thereby also orders of magnitude faster. 

\subsection{6D Planar Quadcopter}

\begin{figure*}[t]
    \centering
    \begin{minipage}[c]{0.05\textwidth}
        \centering
        \rotatebox{90}{$p(x$, $y$)}
    \end{minipage}%
    % position row
    \begin{minipage}[c]{0.95\textwidth}
    %% pdf pics
    \begin{subfigure}[b]{0.130\linewidth}
        \centering
        \includegraphics[width=\linewidth]{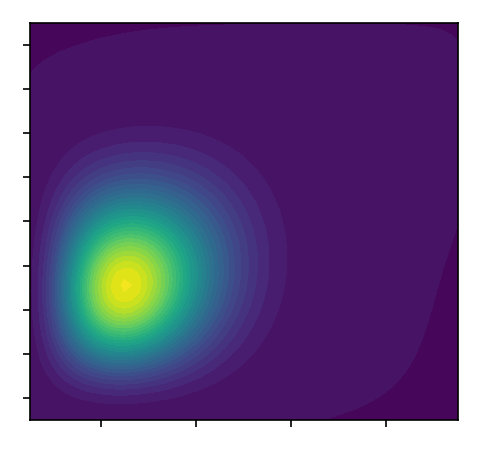}
    \end{subfigure}%
    \begin{subfigure}[b]{0.130\linewidth}
        \centering
        \includegraphics[width=\linewidth]{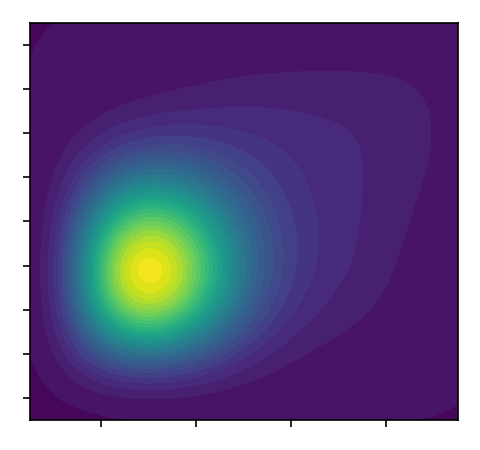}
    \end{subfigure}%
    \begin{subfigure}[b]{0.130\linewidth}
        \centering
        \includegraphics[width=\linewidth]{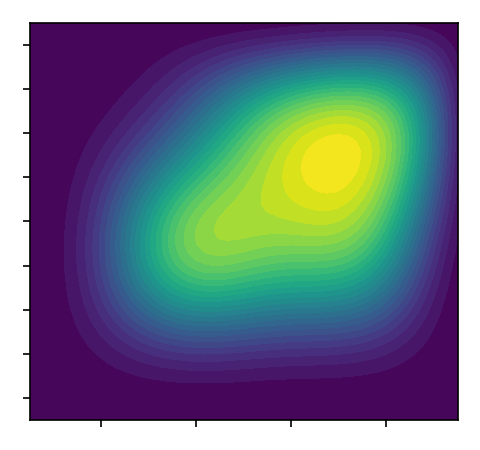}
    \end{subfigure}%
    \begin{subfigure}[b]{0.130\linewidth}
        \centering
        \includegraphics[width=\linewidth]{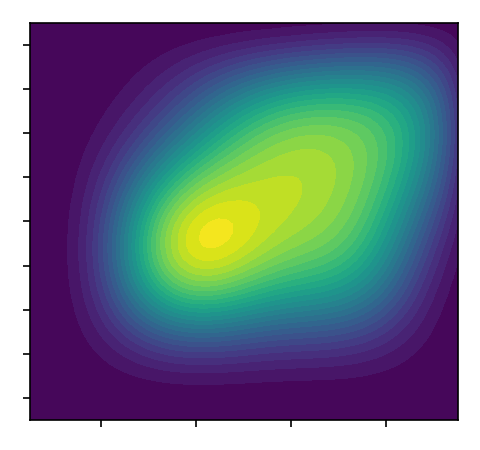}
    \end{subfigure}%
    %% mc pics
    \begin{subfigure}[b]{0.121\linewidth}
        \centering
        \includegraphics[width=\linewidth]{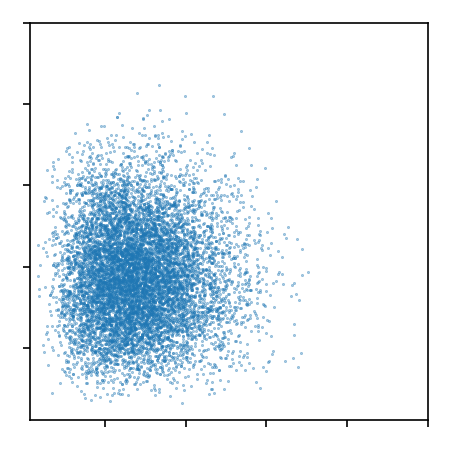}
    \end{subfigure}%
    \begin{subfigure}[b]{0.121\linewidth}
        \centering
        \includegraphics[width=\linewidth]{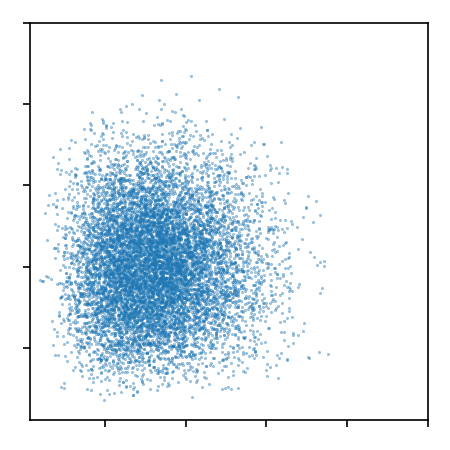}
    \end{subfigure}%
    \begin{subfigure}[b]{0.121\linewidth}
        \centering
        \includegraphics[width=\linewidth]{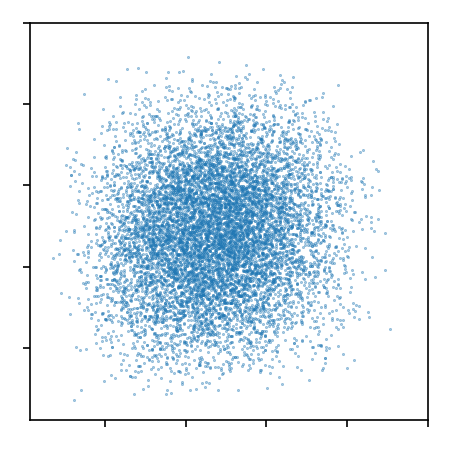}
    \end{subfigure}%
    \begin{subfigure}[b]{0.121\linewidth}
        \centering
        \includegraphics[width=\linewidth]{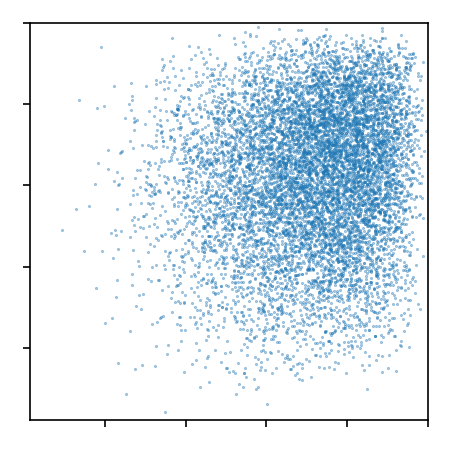}
    \end{subfigure}%
    \end{minipage}
    \centering
    \\
    % angular position row
    \begin{minipage}[c]{0.05\textwidth}
        \centering
        \rotatebox{90}{$p(q, r)$}
    \end{minipage}%
    %% pdf pics
    \begin{minipage}[c]{0.95\textwidth}
    \begin{subfigure}[b]{0.130\linewidth}
        \centering
        \includegraphics[width=\linewidth]{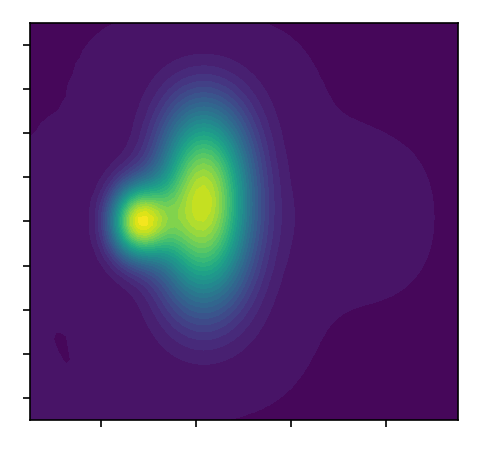}
    \end{subfigure}%
    \begin{subfigure}[b]{0.130\linewidth}
        \centering
        \includegraphics[width=\linewidth]{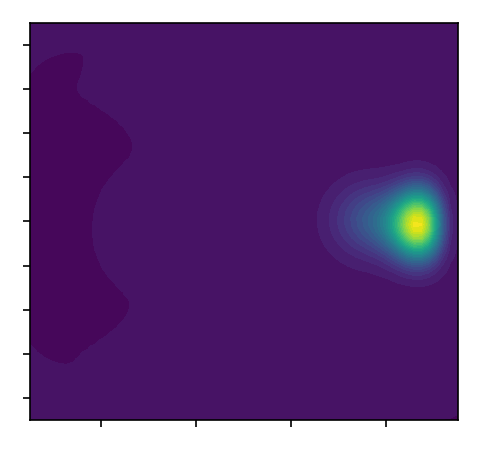}
    \end{subfigure}%
    \begin{subfigure}[b]{0.130\linewidth}
        \centering
        \includegraphics[width=\linewidth]{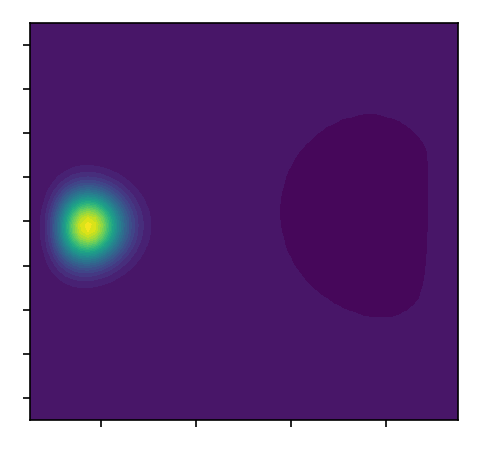}
    \end{subfigure}%
    \begin{subfigure}[b]{0.130\linewidth}
        \centering
        \includegraphics[width=\linewidth]{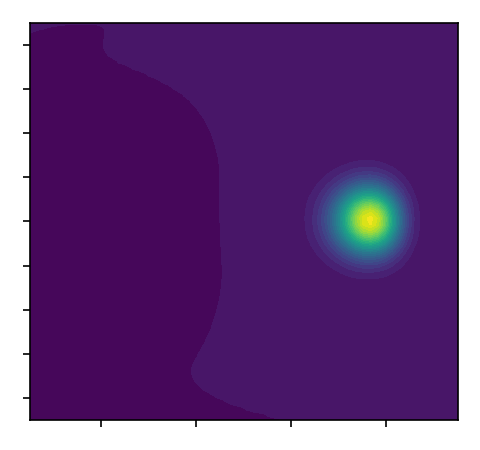}
    \end{subfigure}%
    %% mc pics
    \begin{subfigure}[b]{0.121\linewidth}
        \centering
        \includegraphics[width=\linewidth]{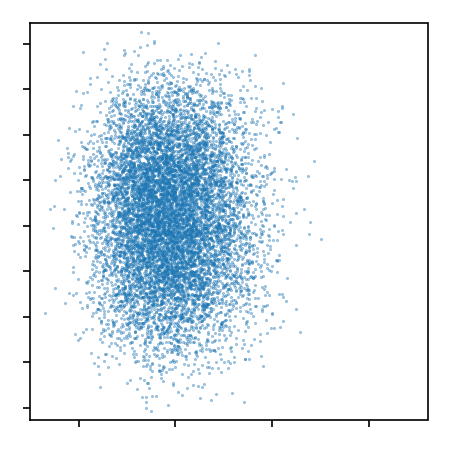}
    \end{subfigure}%
    \begin{subfigure}[b]{0.121\linewidth}
        \centering
        \includegraphics[width=\linewidth]{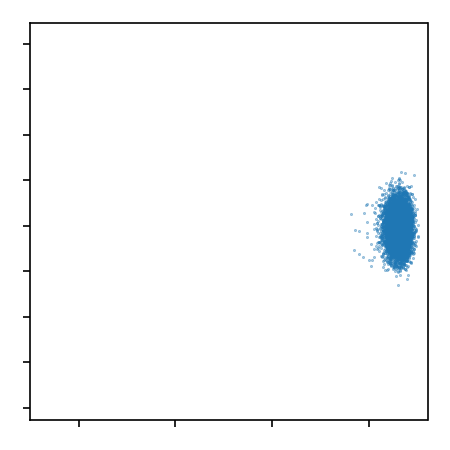}
    \end{subfigure}%
    \begin{subfigure}[b]{0.121\linewidth}
        \centering
        \includegraphics[width=\linewidth]{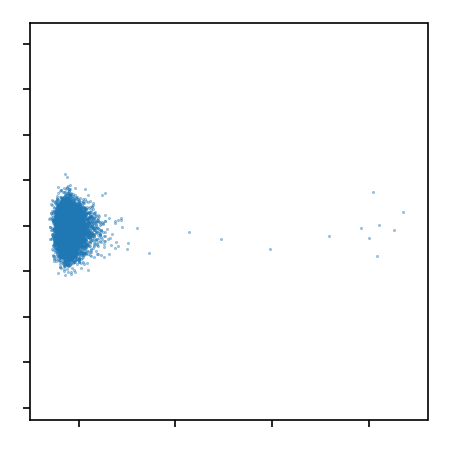}
    \end{subfigure}%
    \begin{subfigure}[b]{0.121\linewidth}
        \centering
        \includegraphics[width=\linewidth]{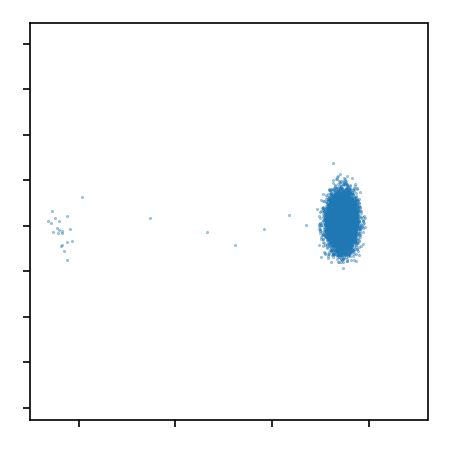}
    \end{subfigure}%
    \end{minipage}
    % velocity row
    \begin{minipage}[c]{0.05\textwidth}
        \centering
        \rotatebox{90}{$p(v_x, v_z)$}
    \end{minipage}%
    \begin{minipage}[c]{0.95\textwidth}
    %% pdf pics
    \begin{subfigure}[b]{0.130\linewidth}
        \centering
        \includegraphics[width=\linewidth]{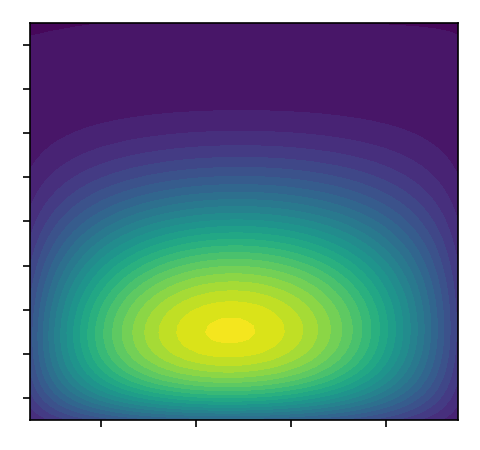}
        \caption*{k = 0}
    \end{subfigure}%
    \begin{subfigure}[b]{0.130\linewidth}
        \centering
        \includegraphics[width=\linewidth]{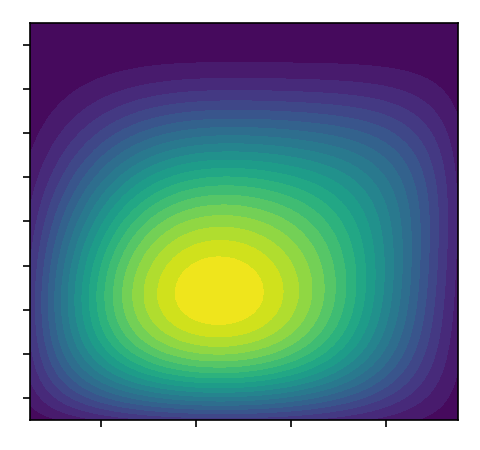}
        \caption*{k = 1}
    \end{subfigure}%
    \begin{subfigure}[b]{0.130\linewidth}
        \centering
        \includegraphics[width=\linewidth]{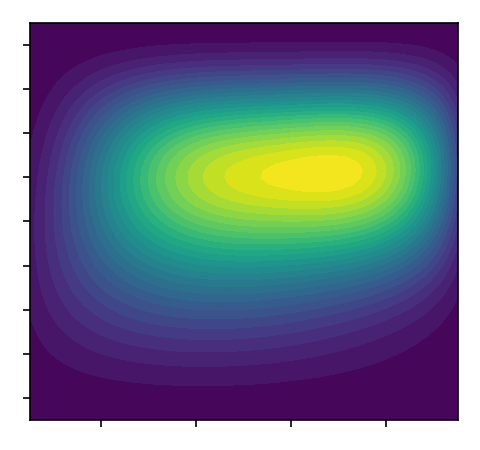}
        \caption*{k = 4}
    \end{subfigure}%
    \begin{subfigure}[b]{0.130\linewidth}
        \centering
        \includegraphics[width=\linewidth]{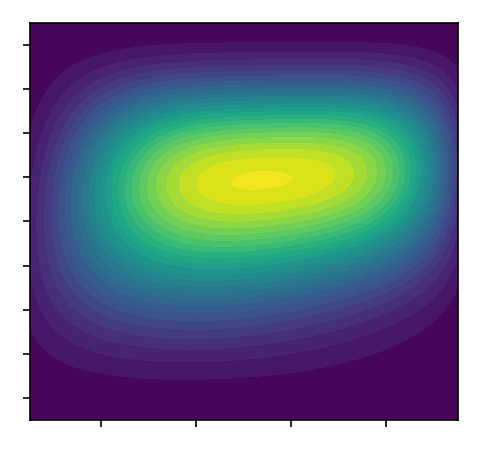}
        \caption*{k = 9}
    \end{subfigure}%
    %% mc pics
    \begin{subfigure}[b]{0.121\linewidth}
        \centering
        \includegraphics[width=\linewidth]{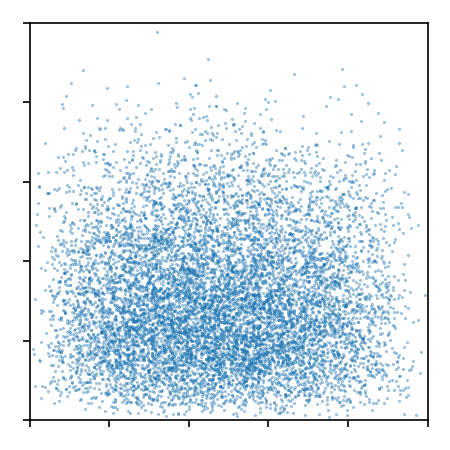}
        \caption*{k = 0}
    \end{subfigure}%
    \begin{subfigure}[b]{0.121\linewidth}
        \centering
        \includegraphics[width=\linewidth]{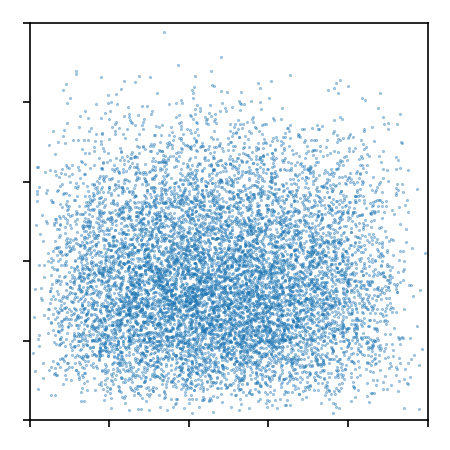}
        \caption*{k = 1}
    \end{subfigure}%
    \begin{subfigure}[b]{0.121\linewidth}
        \centering
        \includegraphics[width=\linewidth]{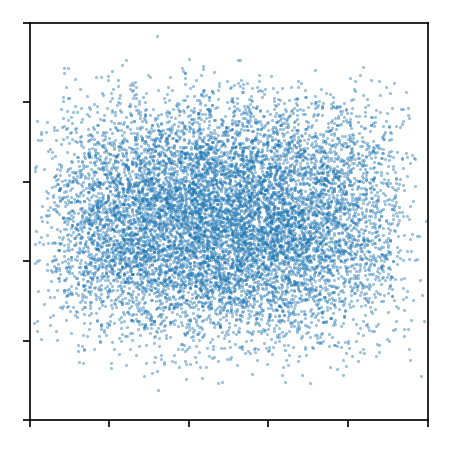}
        \caption*{k = 4}
    \end{subfigure}%
    \begin{subfigure}[b]{0.121\linewidth}
        \centering
        \includegraphics[width=\linewidth]{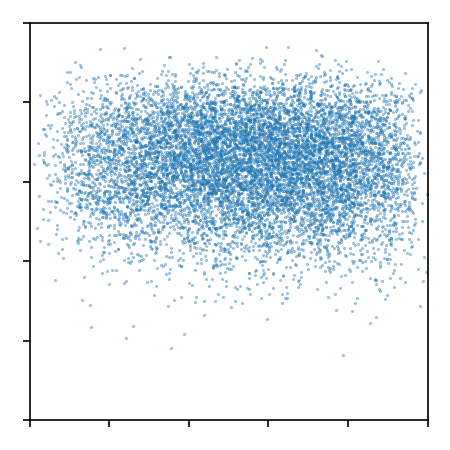}
        \caption*{k = 9}
    \end{subfigure}%
    \end{minipage}
    \caption{Visual Comparison for 12D Quadcopter system. (Left) Learned model. (Right) Monte Carlo simulation (ground truth).
    } \label{fig:12d}
    \label{fig:quadcopter}
\end{figure*}
We analyzed the performance of the model for learning a 6D quadcopter system subject to a state-feedback controller. The quadcopter has non-linear dynamics with additive Gaussian process noise. The state-space is defined as $\px = [x, y, v_x, v_y, \rho, \nu]$ where $x, y$ are 2D planar position components, $\rho, \nu$ are the angular states and $v_x, v_y$ are the planar velocity components. Specifically A state-space transformation was used to learn the Markov process and propagate beliefs in $\uspace^d$. 
To illustrate the power of the sparse model, we chose a relatively small $n=17$ resulting in only 986 parameters for $p(\px' | \px)$ and 493 parameters for $p(\px_0)$.
The initial belief $p(\px_0)$ was trained using 4k data points, and $p(\px' | \px)$ was trained using 40k data points. The full joint belief was propagated; however, for visualization, only pair-wise (analytical) marginal distributions are shown. 

Fig. \ref{fig:quadcopter} shows the propagated marginal beliefs. As can be seen, the RF \sos{} model is able to capture the belief trends of the full 6D system. Similar to mixture models, the optimization is able to successfully allocate basis functions to achieve an expression of the density with very small memory footprint.

\begin{table*}[t]
\centering
\caption{Accuracy comparison to the deep Normalizing Flow model. Average (test) log-likelihood values are shown as mean (std).}
\scriptsize
\resizebox{\textwidth}{!}{%
\begin{tabular}{l|ccccccc}
\hline
Experiment & $k{=}1$ & 3 & 5 & 7 & 9 & 11 & 13 \\
\hline
4D Cartpole (SoS) &
-2.01 (0.0045) & -2.79 (0.0011) & -3.31 (0.0017) & -3.77 (0.0016) &
-4.49 (0.0026) & -5.99 (0.017) & -8.60 (0.11) \\

4D Cartpole (NF) &
-1.88 (0.0019) & -2.49 (0.011) & -3.12 (0.027) & -3.77 (0.056) &
-4.32 (0.050) & -4.86 (0.085) & -5.41 (0.10) \\
\hline

6D Quadrotor (SoS) &
-3.23 (0.22) & -4.56 (0.57) & -5.61 (0.84) & -6.42 (0.82) &
-6.82 (0.92) & -7.06 (1.0) & -7.52 (1.2) \\

6D Quadrotor (NF) &
-4.01 (0.0029) & -5.85 (0.013) & -7.04 (0.033) & -7.93 (0.017) &
-8.69 (0.045) & -9.28 (0.031) & -9.88 (0.076) \\
\hline

6D Dubin's Car w/ Trailer (SoS) &
-4.10 (1.1) & -5.60 (0.79) & -5.84 (0.84) & -6.33 (0.11) &
-6.76 (0.15) & -6.97 (0.16) & -6.97 (0.19) \\

6D Dubin's Car w/ Trailer (NF) &
-3.24 (0.023) & -2.96 (0.049) & -2.28 (0.054) & -3.64 (0.056) &
-4.91 (0.058) & -5.64 (0.085) & -6.07 (0.086) \\
\hline

12D Quadcopter (SoS) &
-9.54 (0.22) & -9.45 (0.024) & -9.36 (0.03) & -9.08 (0.03) &
-9.20 (0.02) & -9.58 (0.01) & -10.3 (0.01) \\

12D Quadcopter (NF) &
-8.15 (0.011) & -7.53 (0.0093) & -8.48 (0.013) & -9.36 (0.12) &
-8.70 (0.0092) & -8.79 (0.0093) & -9.05 (0.010) \\

\hline
\end{tabular}}
\label{tab:nf_comparison}
\end{table*}

\subsection{12D Full-state Quadcopter}
To test the ability for the proposed model to scale to very high-dimensional systems, we performed a propagation experiment on a full 12D Quadcopter model. For this experiment $n=20$ resulting in 1760 parameters for $p(\px' | \px)$ and 880 parameters for $p(\px_0)$. The results can be seen in Fig. \ref{fig:12d} where the position marginals $p(x, y)$, angular rate marginals $p(q, r)$ and velocity marginals $p(v_x, v_y)$ are shown. The RF \sos{} model is able to capture the
general behavior of the belief. In particular, due to a stabilization controller, the oscillation of the angular rates of the quadcopter (seen in the $p(q, r)$ marginals is captured by the belief propagation.

%Compared to the 6D experiment, the RF model does struggle to attain the same fidelity for capturing complex beliefs. This could be due to using higher regularization in order to avoid potential numerical instability with large exponents $\alpha$ and $\beta$. However, these results seem consistent with mixture models: for higher dimensions, the representational capacity decreases with a fixed number of bases. 

\subsection{Comparison to Conditional Deep Generative models}

This section analyzes the efficacy of the method against state-of-the-art deep neural network models capable of learning general conditional state-transition distributions. Such deep network models struggle to propagate belief densities, therefore a particle set is used as the belief approximation. The initial state data is propagated via sampling, then to compare accuracy, a GMM is fitted to each particle belief. In particular, we compare to a conditional masked-affine autoregressive normalizing flow \cite{papamakarios2017masked} with 8 layers each of 128 neurons (1760 total parameters). All \sos{} models used $n=17$ except for the 12D system which used $n=20$. Each experiment was performed 12 times on a fixed train/test data set with randomized training seeds; the mean average log-likelihood (higher is better) with variance in brackets is shown. The GP-based methods generally ran out of memory, and hence, are omitted from these experiments.

As can be seen in Table \ref{tab:nf_comparison}, for moderate dimensions, the SoS method performs comparably to the NF method, even showing benefits in the 6D Quadrotor experiment. However, for the 6D Dubin’s Car experiment, NF outperforms SoS since the transition distribution is very sharp (due to multiplicative noise), and thus is not as well captured by the Beta PDF basis functions used in SoS. This warrants future work on the efficacy of certain basis functions for modeling sharper transition distributions. Additionally, SoS performs slightly worse on 12D, showing that deep methods may currently possess better scalability or representation capacity.

\section{CONCLUSION}
\label{sec:conclusion}
We present a novel approach to modeling Markov processes that (i) can learn general Markov processes, (ii) permit analytical belief propagation, and (iii) achieve a sparse parameter representation, allowing such models to scale to high-dimensional systems. Through an in-depth theoretical analysis, we find that Sum-of-Squares representations alone are highly restrictive for modeling valid conditional distributions while maintaining analytical tractability. The proposed Rational Factor form bypasses such restrictions, modeling valid Markov process models for any choice of basis functions. Our empirical experiments confirm the proposed model's ability to leverage sparsity to scale to systems of up to 12 dimensions.

\ackaccepted{This work was supported in part by the National Science Foundation (NSF) Center for Autonomous Air Mobility \& Sensing (CAAMS) under award 2137269.}

\bibliography{bibliography}

\clearpage
\appendix
\thispagestyle{empty}

% Supplementary material: To improve readability, you must use a single-column format for the supplementary material.
\onecolumn
\aistatstitle{Learning Markov Processes as Sum-of-Square Forms for Analytical Belief Propagation: \\
Supplementary Materials}
\section{PROOFS}

\subsection{Proof of Lemma \ref{lem:fixed_point}}
\begin{proof}
Assuming $\pphi_g = \pphi$, the constraint \eqref{eq:rff_sos} can be expanded into double-sum form

\begin{align}
    &\int_{\px'} \Big( \sum_{k=1}^n \sum_{l=1}^n r_{k, l} \phi_k(\px') \phi_l(\px') \Big) %\notag \\ 
    \Big(\sum_{i=1}^n \sum_{j=1}^n q_{i, j} \phi_i(\px) \psi_i(\px') \phi_j(\px) \psi_j(\px') \Big) d\px' \notag \\
    =& \sum_{i=1}^n \sum_{j=1}^n r_{i, j} \phi_i(\px) \phi_j(\px). \label{eq:app:big_expand}
\end{align}
Rearranging the terms in the integrand and factoring out functions of $\px$, the RHS of \eqref{eq:app:big_expand} becomes
\begin{align}
    &\sum_{i=1}^n \sum_{j=1}^n q_{i, j} \phi_i(\px) \phi_j(\px) \sum_{k=1}^n \sum_{l=1}^n r_{k, l} \; e_{k,l,i,j} \label{eq:app:rhs}
\end{align}
where 
\begin{equation}
    e_{k,l,i,j} = \int_{\px'} \phi_k(\px') \phi_l(\px') \psi_i(\px') \psi_j(\px') d\px'. \label{eq:app:e_def}
\end{equation}
The RHS and LHS of \eqref{eq:app:big_expand} therefore result in linear combinations of bilinear products of the elements in $\pphi(\px)$. Thus, for equality, it is sufficient to match like coefficients of corresponding basis function pairs, i.e. $\phi_i(\px)\phi_j(\px)$:
\begin{align}
    r_{i, j} = q_{i, j} \sum_{k=1}^n \sum_{l=1}^n r_{k, l} e_{k, l, i, j}. \label{eq:app:coeff_match}
\end{align}
Using the matrix vectorization notation of $\pR$ and $\pQ$, we can see that \eqref{eq:app:coeff_match} exactly matches \eqref{eq:fixed_point} when the $((i, j), (k, l))$ element of $\pE$ is defined as \eqref{eq:app:e_def}. In other words, \eqref{eq:app:e_def} corresponds to a four-dimensional tensor when flattened into a matrix with a $(i, j)$ multi-row-index and $(k, l)$ multi-column-index. 
\end{proof}

\section{ADDITIONAL EXPERIMENTAL RESULTS}
%\subsection{6D Planar Quadcopter results (fixed typos)}
%Figure \ref{fig:6d_corrected} shows the corrected results (namely using the correct densities for $p(\rho, \nu)$ for timesteps $k=0, 1, 4$.
%\input{latex_figures/consolidated_exp_6d_corrected}

\subsection{Full 6D Experiment}
Figure \ref{fig:6d_full_evolution} shows the full evolution of the 6D quadcopter system for three marginals.
\begin{figure*}[h!]
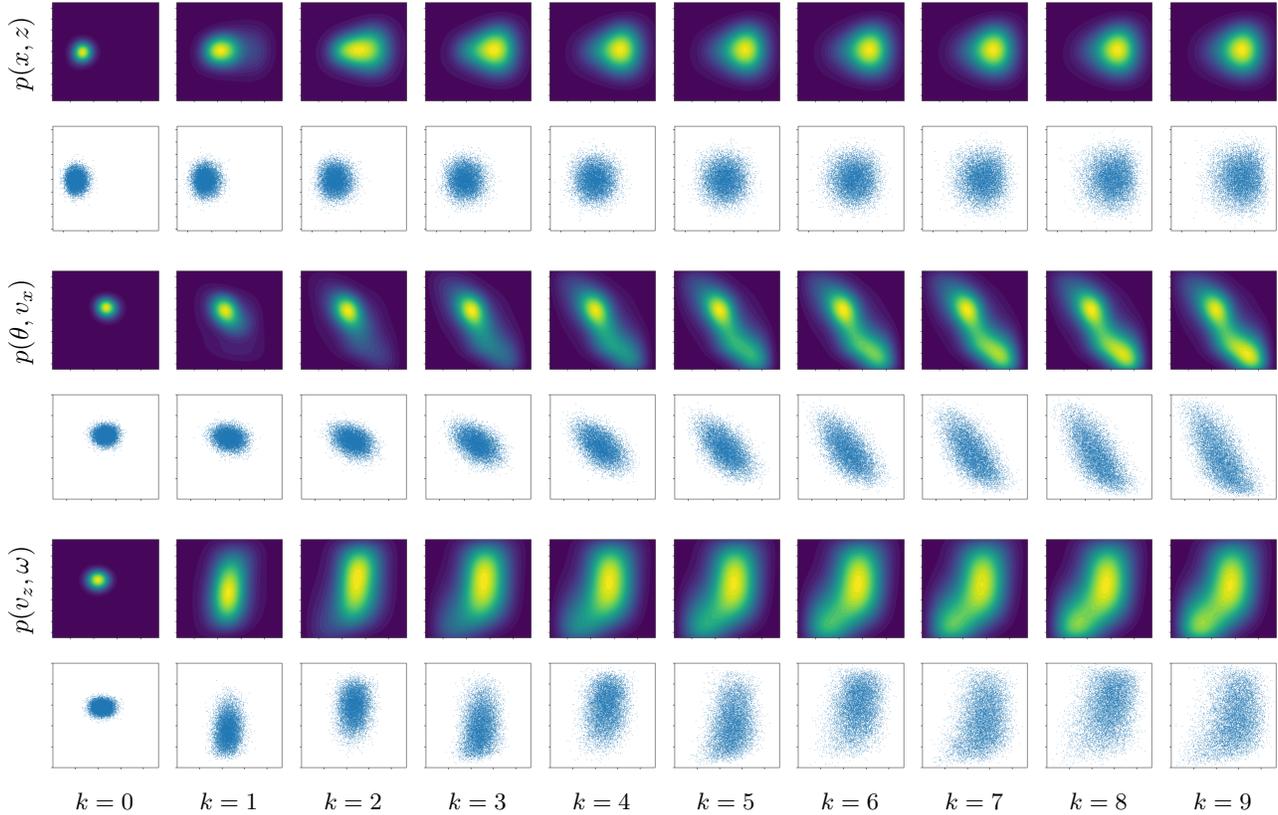

    \centering
    % TikZ foreach required
    % \usepackage{pgffor}
    % \usepackage{graphicx}
    % \usepackage{subcaption}

    % ======================= 1. POSITION MARGINAL =======================
    % Row 1: PDFs
    \begin{minipage}[c]{0.04\textwidth}
        \centering
        \rotatebox{90}{$p(x, z)$}
    \end{minipage}%
    \begin{minipage}[c]{0.96\textwidth}
        \foreach \i/\t in {0/0,1/1,2/2,3/3,4/4,5/5,6/6,7/7,8/8,9/9} {
            \begin{subfigure}[b]{0.095\linewidth}
                \centering
                \includegraphics[width=\linewidth]{figures/quad_6D/pdf/px_pz_t\t.png}
                %\caption*{$k = \i$}
            \end{subfigure}%
        }
    \end{minipage}
    \vspace{1mm}\\
    % Row 2: MC
    \begin{minipage}[c]{0.04\textwidth}
        \centering
        % empty minipage for alignment
        \hspace{5mm}
    \end{minipage}%
    \begin{minipage}[c]{0.96\textwidth}
        \foreach \i/\t in {0/0,1/1,2/2,3/3,4/4,5/5,6/6,7/7,8/8,9/9} {
            \begin{subfigure}[b]{0.095\linewidth}
                \centering
                \includegraphics[width=\linewidth]{figures/quad_6D/mc/px_pz_t\t.png}
                %\caption*{$k = \i$}
            \end{subfigure}%
        }
    \end{minipage}
    \vspace{3mm}\\

    % ======================= 2. ANGULAR MARGINAL =======================
    % Row 3: PDFs
    \begin{minipage}[c]{0.04\textwidth}
        \centering
        \rotatebox{90}{$p(\theta, v_x)$}
    \end{minipage}%
    \begin{minipage}[c]{0.96\textwidth}
        \foreach \i/\t in {0/0,1/1,2/2,3/3,4/4,5/5,6/6,7/7,8/8,9/9} {
            \begin{subfigure}[b]{0.095\linewidth}
                \centering
                \includegraphics[width=\linewidth]{figures/quad_6D/pdf/theta_vx_t\t.png}
                %\caption*{$k = \i$}
            \end{subfigure}%
        }
    \end{minipage}
    \vspace{1mm}\\
    % Row 4: MC
    \begin{minipage}[c]{0.04\textwidth}
        \centering
        \hspace{5mm}
    \end{minipage}%
    \begin{minipage}[c]{0.96\textwidth}
        \foreach \i/\t in {0/0,1/1,2/2,3/3,4/4,5/5,6/6,7/7,8/8,9/9} {
            \begin{subfigure}[b]{0.095\linewidth}
                \centering
                \includegraphics[width=\linewidth]{figures/quad_6D/mc/theta_vx_t\t.png}
                %\caption*{$k = \i$}
            \end{subfigure}%
        }
    \end{minipage}
    \vspace{3mm}\\

    % ======================= 3. VELOCITY MARGINAL =======================
    % Row 5: PDFs
    \begin{minipage}[c]{0.04\textwidth}
        \centering
        \rotatebox{90}{$p(v_z, \omega)$}
    \end{minipage}%
    \begin{minipage}[c]{0.96\textwidth}
        \foreach \i/\t in {0/0,1/1,2/2,3/3,4/4,5/5,6/6,7/7,8/8,9/9} {
            \begin{subfigure}[b]{0.095\linewidth}
                \centering
                \includegraphics[width=\linewidth]{figures/quad_6D/pdf/vz_omega_t\t.png}
                %\caption*{$k = \i$}
            \end{subfigure}%
        }
    \end{minipage}
    \vspace{1mm}\\
    % Row 6: MC
    \begin{minipage}[c]{0.04\textwidth}
        \centering
        \hspace{5mm}
    \end{minipage}%
    \begin{minipage}[c]{0.96\textwidth}
        \foreach \i/\t in {0/0,1/1,2/2,3/3,4/4,5/5,6/6,7/7,8/8,9/9} {
            \begin{subfigure}[b]{0.095\linewidth}
                \centering
                \includegraphics[width=\linewidth]{figures/quad_6D/mc/vz_omega_t\t.png}
                \caption*{$k = \i$}
            \end{subfigure}%
        }
    \end{minipage}

    \caption{Visual comparison for the 6D quadcopter system across timesteps $k = 0 \dots 9$. 
    Beliefs are shown with corresponding Monte Carlo ground truth below.}
    \label{fig:6d_full_evolution}
\end{figure*}

\subsection{Full 12D Experiment}
Figure \ref{fig:12d_full_evolution} shows the full evolution of the 12D stabilizing quadcopter six different marginals.

\begin{figure*}[t!]
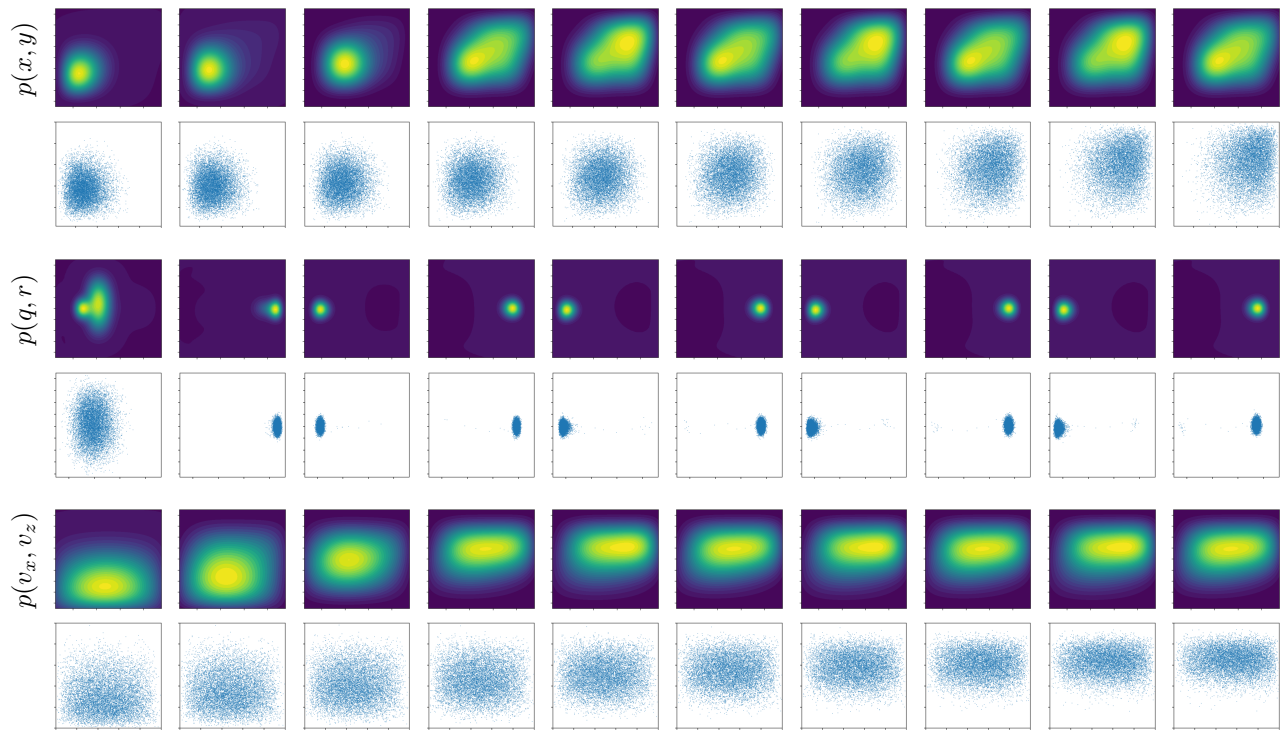

    \centering
    % TikZ foreach required
    % \usepackage{pgffor}
    % \usepackage{graphicx}
    % \usepackage{subcaption}

    % Row 1: PDFs
    \begin{minipage}[c]{0.04\textwidth}
        \centering
        \rotatebox{90}{$p(x, y)$}
    \end{minipage}%
    \begin{minipage}[c]{0.96\textwidth}
        \foreach \i/\t in {0/0,1/1,2/2,3/3,4/4,5/5,6/6,7/7,8/8,9/9} {
            \begin{subfigure}[b]{0.095\linewidth}
                \centering
                \includegraphics[width=\linewidth]{figures/quad_12D/pdf/px_py_t\t.png}
                %\caption*{$k = \i$}
            \end{subfigure}%
        }
    \end{minipage}\\
    % \vspace{1mm}\\
    % Row 2: MC
    \begin{minipage}[c]{0.04\textwidth}
        \centering
        % empty minipage for alignment
        \hspace{5mm}
    \end{minipage}%
    \begin{minipage}[c]{0.96\textwidth}
        \foreach \i/\t in {0/0,1/1,2/2,3/3,4/4,5/5,6/6,7/7,8/8,9/9} {
            \begin{subfigure}[b]{0.095\linewidth}
                \centering
                \includegraphics[width=\linewidth]{figures/quad_12D/mc/px_py_t\t.png}
                %\caption*{$k = \i$}
            \end{subfigure}%
        }
    \end{minipage}
    \vspace{2mm}\\

    % Row 3: PDFs
    \begin{minipage}[c]{0.04\textwidth}
        \centering
        \rotatebox{90}{$p(q, r)$}
    \end{minipage}%
    \begin{minipage}[c]{0.96\textwidth}
        \foreach \i/\t in {0/0,1/1,2/2,3/3,4/4,5/5,6/6,7/7,8/8,9/9} {
            \begin{subfigure}[b]{0.095\linewidth}
                \centering
                \includegraphics[width=\linewidth]{figures/quad_12D/pdf/q_r_t\t.png}
                %\caption*{$k = \i$}
            \end{subfigure}%
        }
    \end{minipage}
    % \vspace{1mm}
    \\
    % Row 4: MC
    \begin{minipage}[c]{0.04\textwidth}
        \centering
        \hspace{5mm}
    \end{minipage}%
    \begin{minipage}[c]{0.96\textwidth}
        \foreach \i/\t in {0/0,1/1,2/2,3/3,4/4,5/5,6/6,7/7,8/8,9/9} {
            \begin{subfigure}[b]{0.095\linewidth}
                \centering
                \includegraphics[width=\linewidth]{figures/quad_12D/mc/q_r_t\t.png}
                %\caption*{$k = \i$}
            \end{subfigure}%
        }
    \end{minipage}
    \vspace{2mm}\\

    \begin{minipage}[c]{0.04\textwidth}
        \centering
        \rotatebox{90}{$p(v_x, v_z)$}
    \end{minipage}%
    \begin{minipage}[c]{0.96\textwidth}
        \foreach \i/\t in {0/0,1/1,2/2,3/3,4/4,5/5,6/6,7/7,8/8,9/9} {
            \begin{subfigure}[b]{0.095\linewidth}
                \centering
                \includegraphics[width=\linewidth]{figures/quad_12D/pdf/vx_vz_t\t.png}
                %\caption*{$k = \i$}
            \end{subfigure}%
        }
    \end{minipage}
    % \vspace{1mm}
    \\
    \begin{minipage}[c]{0.04\textwidth}
        \centering
        \hspace{5mm}
    \end{minipage}%
    \begin{minipage}[c]{0.96\textwidth}
        \foreach \i/\t in {0/0,1/1,2/2,3/3,4/4,5/5,6/6,7/7,8/8,9/9} {
            \begin{subfigure}[b]{0.095\linewidth}
                \centering
                \includegraphics[width=\linewidth]{figures/quad_12D/mc/vx_vz_t\t.png}
            \end{subfigure}%
        }
    \end{minipage}
    \vspace{2mm}\\

    \caption{Visual comparison for the 12D stabilizing quadcopter system across timesteps $k = 0 \dots 9$. 
    Beliefs are shown with corresponding Monte Carlo ground truth below. The full state definition is $\px = [x, y, z, \phi, \theta, \psi, v_x, v_y, v_z, \omega_p, \omega_q, \omega_r]$ with position $(x,y,z)$ and Euler angles $(\phi, \theta, \psi)$. Position marginals $p(x, y)$, angular velocity marginals $p(r, q)$ and translational velocity marginals $p(v_x, v_z)$ are shown.}
    \label{fig:12d_full_evolution}
\end{figure*}

\subsection{Experimental Details}
Each state-space was converted to $\uspace^d$ via mapping each state-space dimension independently through a Gaussian cumulative distribution function. For more technical details, as well as how each transformation can be chosen, refer to \cite{amorese2026universal}. The full state-space dynamics equations can be found in the supplementary code in \cite{github}, with the used parameters.

Regularization loss was applied in the form
\begin{equation}
    \mathcal{L}_{reg} = c_{reg} \sum_i \sum_m \big(\alpha_{i,m}^2 + \beta_{i, m}^2 \big).
\end{equation}
A regularization weight of $c_{reg}=10^{-4}$ with $\alpha, \beta \in [0.4, 400.0]$ 
%. The 12D experiment required a higher regularization loss of $c_{reg}=1e-2$ to prevent blow up of large exponents of \eqref{eq:beta_pdf}. 
%For the 6D experiment, all $\alpha, \beta \in [0.1, 400.0]$ for the transition model and $\alpha, \beta \in [0.4, 100.0]$ for the initial state model. For the 12D experiments, all $\alpha, \beta \in [0.2, 50.0]$ for both models.

All models were trained on GPU hardware, with training split between a NVIDIA RTX A5000 and NVIDIA GeForce RTX 2070. All models were trained in under two hours, and belief propagation was performed on the CPU in under three seconds.

\end{document}